\documentclass{article} % For LaTeX2e
\usepackage{iclr2027_conference,times}

\usepackage{amsmath,amsfonts,bm}

\def\eqref#1{equation~\ref{#1}}
\def\1{\bm{1}}

\DeclareMathAlphabet{\mathsfit}{\encodingdefault}{\sfdefault}{m}{sl}
\SetMathAlphabet{\mathsfit}{bold}{\encodingdefault}{\sfdefault}{bx}{n}

\newcommand{\E}{\mathbb{E}}

\DeclareMathOperator*{\argmax}{arg\,max}

\usepackage{hyperref}
\usepackage{url}

\usepackage{amsthm}
\usepackage{algorithm}
\usepackage{algorithmic}
\usepackage{multirow}
\usepackage{booktabs}
\usepackage{tabularx}
\usepackage{graphicx} 
\usepackage{subcaption}
\usepackage{xcolor}
\usepackage{array}
\usepackage{longtable}
\usepackage{placeins}
\usepackage{enumitem} % spaces in enum/item

\newtheorem{theorem}{Theorem}[section]
\newtheorem{proposition}{Proposition}[section]
\newtheorem*{proposition*}{Proposition} % restate theorem
\newtheorem{assumption}{Assumption}[section]
\newtheorem{lemma}{Lemma}[section]
\newtheorem{corollary}{Corollary}[section]
\newtheorem{remark}{Remark}[section]

\newif\ifarxiv
\arxivtrue
\newcommand{\pist}{\pi_{\theta}}

\newcommand{\dtv}{d_{\mathrm{TV}}}
\newcommand{\cA}{\mathcal{A}} 
\DeclareMathOperator{\supp}{supp}  

\title{Does This Action Still Explain the Task? Reverse Scoring for Diffusion Language Model Agents}

\author{%
\begin{minipage}{\textwidth}
\centering
\begin{tabular}{c@{\hspace{3em}}c@{\hspace{3em}}c}
\textbf{Jiacheng Qiu}\textsuperscript{1,2}
&
\textbf{Christopher E. Mower}\textsuperscript{1}
&
\textbf{Jan Peters}\textsuperscript{2}
\end{tabular}
\\[1.5em]
\begin{tabular}{c@{\hspace{3em}}c}
\textbf{Haitham Bou-Ammar}\textsuperscript{1,3}
&
\textbf{Matthieu Zimmer}\textsuperscript{1}
\end{tabular}
\\[1.2em]
{\normalfont\small
\textsuperscript{1} Huawei Noah's Ark Lab\\
\textsuperscript{2} Technical University of Darmstadt\\
\textsuperscript{3} UCL Centre for AI
}
\end{minipage}%
}

\ifarxiv
\iclrfinalcopy
\fi

\begin{document}

\maketitle

\ifarxiv
\lhead{Preprint}
\fi

\begin{abstract}
Diffusion-based large language models (dLLMs) promise to break the sequential latency bottleneck of autoregressive agents through parallel decoding, but recent evaluations show this efficiency does not transfer to embodied agentic competence: dLLM-backed agents repeatedly fall into retry loops, re-issuing an action long after it has failed. 
We give a mechanistic account of this failure and a training-free remedy.
We trace the retry loop to the adaptivity of masked decoding: the sampler commits the positions it is most confident about and defers the uncertain ones, and at a failure state the context already offers a confident fill for the deferred decision, i.e. the failed action itself, so the retry is committed without the failure feedback ever being confronted.
We model the resulting distortion of the action distribution as a \emph{task-blind corruption}: contextually salient actions (e.g., the action just taken) receive inflated probability by a factor that depends on the state and the action but not on the task.
Under this model, we analyse an invariance proposition: the task-blind factor cancels exactly from the \emph{reverse conditional}, i.e. the likelihood of the task given the state and a candidate action, which coincides with the task posterior of an idealized uncorrupted model.
Masked dLLMs evaluate the reverse conditional natively, unlike autoregressive models, by masking the task tokens and denoising, at the cost of a few parallel passes per candidate.
We instantiate the rule as \textbf{Reflect Reverse} and evaluate it on four multi-turn embodied benchmarks, where it improves task success and progression rates over forward-scoring baselines.
\end{abstract}

\section{Introduction}
\label{sec:intro}

An agent that samples its next action from a language model inherits the pathologies of the distribution it samples from.
Diffusion language models (dLLMs) \citep{li2022diffusionlmimprovescontrollabletext,lou2023discrete,gong2025scaling} are attractive agent backbones: they generate entire spans in a handful of parallel denoising steps, promising to break the sequential latency bottleneck of autoregressive decoding \citep{christianos2023panguagentfinetunablegeneralistagent}.
But the distribution a dLLM samples from has a distinctive, structural pathology in embodied settings: it over-weights actions that are already salient in the context, by an amount that depends on the state and the action but not on the task.
The bias is present at every decision point, however it is clearly visible on  failure states: at a state where the previous action has just failed, the model keeps proposing that same action, verbatim, as if the most re-usable prediction were whatever already sits in context \citep{holtzman2019curious}.

Why should this be structural rather than a matter of knowledge?
%\citet{ni2026flexibilitytraprethinkingvalue} identify a failure mode of masked decoding that, we argue, reappears in a sharper form in agents.
Masked decoding is \emph{adaptive}: the sampler commits the masked positions it is most confident about and defers the rest \citep{Nie2025LargeLD}, and \citet{ni2026flexibilitytraprethinkingvalue} show that the deferred positions are the ``forking'' ones where several continuations remain viable: by the time the sampler returns to them, the committed context has already resolved the choice.
At a failure state, \emph{which action now?} is such a fork, and the context already offers a confident fill for it: the failed action itself, verbatim in $s$.
The retry is thus committed without the open decision ever being sampled: the failure feedback is bridged, not confronted (Section~\ref{sec:related} develops the mechanism).
%An open decision degenerates into a retrospective alignment with a pre-determined gap, a phenomenon they term \emph{entropy degradation}.
%At a failure state, \emph{which action now?} is exactly such a fork.
%Our agent thinks before acting, i.e. the span to be denoised is a chain of thought followed by the action, but the thinking does not confront the fork: every position, in the reasoning or in the action, can be deferred until the committed context makes it confident, and for the action positions the context already contains a confident fill, the failed action itself, every token of which sits verbatim in $s$.
%In reasoning, the cost of this trap is lost solution coverage; in agents, it is the retry loop.
%The mechanism does not hinge on parallel decoding: it persists whenever the sampler chooses \emph{which} position to commit by confidence, even at one token per step, whereas an autoregressive model has no such choice and must decide the first token of the new action while the failure feedback is the newest evidence.
Resampling at a stuck state shows the knowledge is there: the candidates include plausible alternatives, yet the distribution still concentrates on the contextually salient one.
The bias does not care which task is being attempted.

\begin{figure*}[t]
    \centering
    \includegraphics[width=1.0\textwidth]{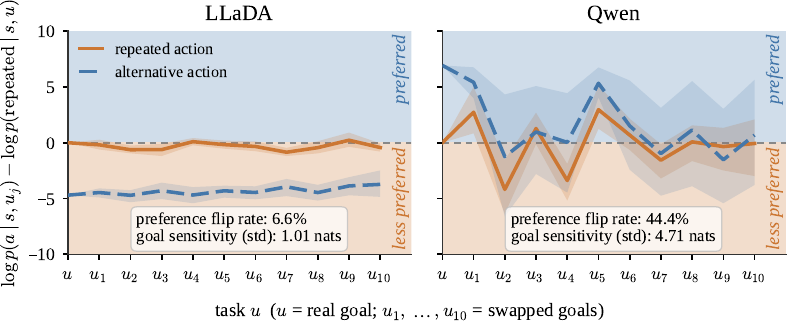}
    \caption{Task-blind corruption analyzed over 208 failed states in ALFWorld.
    }
    \label{fig:corruption}
%1. The repeated (failed) action is essentially task-blind. Its log-probability moves by <~1 nat under any goal swap, in LLaDA — the agent keeps committing to the action that already failed, regardless of what the goal becomes.
%2. The models differ sharply in goal sensitivity. LLaDA flips its preferred action in only 5.8% of state×swap cells (σ = 1.39 nats); Qwen flips 38.7% (σ = 4.97 nats). So task-blind commitment to the failed action is much stronger for LLaDA.
%3. LLaDA barely reweights at all: its alternative (goal-directed) action sits ~5–6 nats below the repeated baseline and stays flat — the model is "stuck" on the repeated action.
%4. Qwen does reweight, but erratically: the alternative swings from −3.5 to +5.9 nats, with a wide band reflecting genuine between-episode heterogeneity (some episodes swing −13, others +4).
\end{figure*}

We make this precise and turn it into a correction with a provable guarantee.
Let $u$ denote the task, $s$ the interaction history, and $a$ a candidate next action. 
Write $\pist(a\mid s,u)$ for the observable model conditional, and $\mu(a\mid s,u)$ for the \emph{faithful reference} conditional: the distribution the model would produce if it had absorbed $s$ correctly. 
Our hypothesis is that the gap between the two is a multiplicative distortion that does not depend on the task (Assumption~\ref{ass:corruption}): the same state--action inflation must explain the model's behavior under every goal.
To verify it empirically, we record the log-probability of the retried ("repeated") and goal-directed ("alternative") actions under the real goal $u$ and ten swapped goals $u_1,\dots,u_{10}$, and plot the mean change relative to the real-task level (Figure~\ref{fig:corruption}). Full details of the task-corruption
procedure are provided in Appendix~\ref{app:taskcorruption}. For both actions, LLaDA \citep{Nie2025LargeLD} is nearly insensitive to goal swaps, and always prefers the failed repeated action, in contrary to the faithful reference represented by Qwen here.

If the corruption is task-blind, it should be invisible to any score that compares candidates \emph{by their task evidence alone}. 
This suggests inverting the question. 
The standard rule asks the model the forward question: \emph{which action for this task?} --- and the corruption survives, because it inflates an action's probability independently of the task. 
We instead ask the reverse question: \emph{does the state, together with this candidate action, still explain the task?}
Every candidate is then scored on the task evidence for it, $\pist(u\mid s,a)$; under the coherence idealization of Section~\ref{sec:corruption}, a factor that carries no task information cannot move this score. 
For a masked dLLM the reverse question is a native operation unavailable easily in autoregressive models: mask the task tokens, make the state and the candidate action visible and denoise. 

Our main theoretical analysis makes this precise (Proposition~\ref{thm:main}): under the corruption and coherence assumptions (Section~\ref{sec:corruption}), the reverse score $\pist(u\mid s,a)$ equals, for every action, the faithful model's task posterior under a task prior that the corruption reweights.
The posterior is proportional to a \emph{task-evidence ratio} $\mu(a\mid s,u)\big/m(a)$: the faithful likelihood of the action under the actual task, divided by its likelihood $m(a)$ under a mixture of tasks.
The corruption cannot survive this comparison: it varies across actions, which is how it distorts the forward ranking, but at each action it is constant across tasks, which is why it cancels from a posterior over tasks.

We instantiate our method as \textbf{Reflect Reverse}: sample candidates with chain of thoughts, score each by the reverse conditional, sample the next action proportionally to the scores. 
On four multi-turn embodied benchmarks and two dLLMs, Reflect Reverse reduces retry-loops and improves task success and progression rates over forward-scoring baselines.

\paragraph{Contributions.}
\begin{itemize}[leftmargin=1.5em, topsep=0pt, itemsep=0pt]
\item A task-blind corruption hypothesis for dLLM action distributions, stated formally and empirically verified (Section~\ref{sec:intro}, Figure~\ref{fig:corruption}).
\item A recovery proposition: the reverse score equals the faithful model's task posterior (Proposition~\ref{thm:main}); we also analyze the reranking's residual bias  (Theorem~\ref{thm:bias}).
\item Reflect Reverse, a training-free action-selection rule that scores candidates by the model's reverse conditional, requiring only a few forward passes per candidate (Section~\ref{sec:method}).
\item Experiments on ALFWorld, ScienceWorld, BabyAI, and Jericho showing improved success rates and progression rates with LLaDA and iLLaDA (Section~\ref{sec:experiments}).
\end{itemize}

\section{Background}
\label{sec:background}

\subsection{Masked diffusion language models}
\label{sec:dllm}

Masked dLLMs \citep{Nie2025LargeLD,austin2021structured,sahoo2024simple,shi2024simplified} define a forward process that progressively masks tokens of a sequence $x=(x_1,\dots,x_n)$ and a reverse process that restores them. At each reverse step, a Transformer with bidirectional attention observes the currently unmasked tokens $x^{o}$ and predicts every masked position simultaneously $
p_\theta(x^{m}\mid x^{o}) \;=\; \prod_{j\in m} p_\theta(x_j\mid x^{o}),$
i.e., the per-step joint over masked tokens is \emph{factorized} given the context. 
Any span of a sequence can be scored against any other span by masking the first and clamping the second, in multiple denoising steps; there is no privileged direction of conditioning \citep{schiff2025simple}. 

\subsection{Task-blind corruption}
\label{sec:corruption}

%Fix a state $s$ and task $u$. 
We model the observed conditional as a faithful conditional distorted by a state--action factor.

\begin{assumption}[Task-blind multiplicative corruption]
\label{ass:corruption}
There exists a function $\varepsilon(s,a)$, depending on the state and the
action but \emph{not} on the task, such that for every task $u$ in the
support of $p(\cdot\mid s)$:
\begin{equation}
\label{eq:corruption}
\pist(a\mid s,u) \;=\; \frac{\mu(a\mid s,u)\,e^{\varepsilon(s,a)}}{Z(s,u)},
\qquad
Z(s,u):=\sum_{a'}\mu(a'\mid s,u)\,e^{\varepsilon(s,a')}.
\end{equation}
\end{assumption}
The content of the assumption is the clause that $\varepsilon$ does not depend on the task: for any two conditionals one can always fit some $\varepsilon$. Figure~\ref{fig:corruption} tests it: at fixed states, the log-ratio $\log\pist(a\mid s,u)-\log\mu$-proxy varies with $(s,a)$ but is flat in $u$.
The factor $\varepsilon(s,a)$ is the structural bias itself: an action whose tokens are salient in $s$ is inflated by the same amount whatever the goal.

\section{Method}
\label{sec:method}

\subsection{Reflect Reverse}
\label{sec:method-algo}

Standard decoding samples an action from $\pist(a\mid s,u)$, i.e. the forward question, but inherits the corruption of Assumption~\ref{ass:corruption}. 
Reflect Reverse inverts the direction of conditioning:
\begin{enumerate}[leftmargin=1.5em, topsep=0pt, itemsep=0pt]
\item \textbf{Propose.} Sample $K$ candidate actions $a_1,\dots,a_K\sim\pist(\cdot\mid s,u)$ with chain of thoughts \citep{wei2022chain} and deduplicate, yielding a candidate set $\mathcal{A}(s)$.

\item \textbf{Reflect.} Score each candidate by the reverse conditional,
estimated by Monte Carlo rollouts \citep{gulrajani2023likelihood,ou2025your}: provide $s$ and $a$, sample $M$ random masks of the task tokens, denoise, and read off the task-token log-likelihoods at the masked positions:
\begin{equation}
\label{eq:score}
\phi(a)
\;:=\;
\frac{1}{M}\sum_{m=1}^{M}\frac{1}{|B_m|}
\sum_{j\in B_m}\log\pist(u_j\mid s,a,\tilde u^{(m)}),
\end{equation}
where $B_m\subseteq\{1,\dots,|u|\}$ is the random subset of task-token positions masked in rollout $m$ and $\tilde u^{(m)}$ is the masked task-token. 
Each term of Eq.\ref{eq:score} is a masked-estimate of the task log-likelihood $\log\pist(u\mid s,a)$, so $\phi(a)$ is a Monte Carlo estimate of the same reverse score.

\item \textbf{Act.} Sample the next action from the corrected policy
\begin{equation}
\label{eq:policy}
\rho(a\mid s,u)
\;:=\;
\frac{\exp(\beta\,\phi(a))}
{\sum_{a'\in\mathcal{A}(s)}\exp(\beta\,\phi(a'))},
\qquad a\in\mathcal{A}(s),
\end{equation}
with inverse temperature $\beta>0$.
\end{enumerate}

The whole procedure is training-free and adds $|\mathcal{A}(s)| \times M$ denoising passes per decision.
Pseudocode is in Appendix \ref{app:pseudo}.

\subsection{Invariance of the reverse score to task-blind corruption}
\label{sec:method-theory}

The reverse score is not an arbitrary heuristic: under the corruption model of Section~\ref{sec:corruption}, it equals the faithful model $\mu$'s own verdict.

\begin{proposition}[Recovery of the faithful task posterior]
\label{thm:main}
Let Assumptions~\ref{ass:corruption} and \ref{ass:coherence} hold. Then, for every state $s$, task $u$ in the support of $p(\cdot\mid s)$, and action $a$,
\begin{equation}
\label{eq:master}
\pist(u\mid s,a)
\;=\;
\frac{\omega(u\mid s)\,\mu(a\mid s,u)}
{\displaystyle\sum_{u'}\omega(u'\mid s)\,\mu(a\mid s,u')},
\end{equation}
where $\omega(u\mid s):=p(u\mid s)/Z(s,u)$ is a positive task weight that does not depend on the  action $a$.
\end{proposition}

The proof is in Appendix~\ref{app:theory_prop}.
The cancellation is the pointwise mutual information (PMI) invariance identified by \citet{holtzman2021surface}.
What is new here is not the identity but its agentic instantiation.
The corruption is not an input-level bias corrected by an external signal \citep{holtzman2021surface,chae2024mitigating} but a decoding-level mechanism.

\paragraph{An exact score for comparing candidates.}
Identity Eq.\ref{eq:master} holds pointwise in $a$, so the pairwise log-margins of the score $\pist(u\mid s,\cdot)$ coincide with those of the faithful task-evidence ratio $\mu(a\mid s,u)\big/\sum_{u'}\omega(u'\mid s)\,\mu(a\mid s,u')$: the score orders and ties candidates exactly as the ratio does.
The remaining factor $\omega(u\mid s)$ is shared by all candidates, so softmax normalization removes it at any temperature (Corollary~\ref{cor:policy-app}): \textbf{the corrected policy $\rho$ is exactly the softmax of the faithful task-evidence ratio over the candidate set, at any $\beta$.}

\paragraph{What the guarantee does and does not cover.}
The guarantee governs the comparison among the candidates on the table, not the table itself: the proposal step still draws candidates from the corrupted $\pist(\cdot\mid s,u)$, so the support $\mathcal{A}(s)$ need not coincide with the faithful model's support.
Moreover, $\rho$ matches the faithful \emph{task-evidence ratio}, not the normalized faithful conditional: writing $m(a):=\sum_{u'}\omega(u'\mid s)\,\mu(a\mid s,u')$ for the task-marginal likelihood of $a$, we have $\rho(a)\propto\mu(a\mid s,u)^{\beta}\big/m(a)^{\beta}$, and the factor $m(a)$ depends on the candidate and is not removed by normalization.
$\rho$ therefore coincides with sampling from $\mu(\cdot\mid s,u)$ restricted to $\mathcal{A}(s)$ only when the marginal is flat across the candidate set (Remark~\ref{rem:ratio-vs-conditional}); what holds unconditionally is the exact recovery of the ratio's ordering and ties.

\paragraph{Why the forward score cannot say this.}
Raw probability has no such invariance: $\pist(a\mid s,u)\propto\mu(a\mid s,u)\,e^{\varepsilon(s,a)}$ reorders freely as $\varepsilon$ varies across candidates.
The reverse question is precisely the one in which a task-blind factor is invisible under our assumptions.
%
%\paragraph{The retry loop, in one line.}
At a failure state, the retried action carries a large $\varepsilon(s,a)$, because it is already in context, so it wins the forward comparison; but its faithful task evidence is low, because it just failed, so it loses the reverse comparison.
%The cancellation removes exactly the mechanism that produces the retry.

\subsection{How large is the residual bias?}
\label{sec:method-bias}

How far is $\rho$ from the faithful reference $\mu(\cdot\mid s,u)$ on the candidate set, and how does that residual bias compare with the bias of forward selection from the corrupted $\pist(\cdot\mid s,u)$ on the same set?
Theorem~\ref{thm:bias} below answers both questions in total variation, $\dtv(P,Q):=\tfrac12\sum_{a}|P(a)-Q(a)|$; the proof, and the sharp range lemma it rests on, are in Appendix~\ref{app:bias}.

Fix a state $s$, a task $u$ in the support of $p(\cdot\mid s)$, and a finite candidate set $\cA\subseteq\supp\mu(\cdot\mid s,u)$, abbreviating $\cA=\mathcal{A}(s)$ when the candidate set is the proposed one.
Write $\mu_{\cA}$ and $\pi_{\cA}$ for the faithful and corrupted conditionals renormalized to $\cA$
and measure the heterogeneity of the candidates through the log-spreads:
\begin{align}
\Lambda_\mu
&\;:=\;
\max_{a\in\cA}\log\mu(a\mid s,u)-\min_{a\in\cA}\log\mu(a\mid s,u),
\label{eq:lam-mu}\\
\Lambda_m
&\;:=\;
\max_{a\in\cA}\log m(a)-\min_{a\in\cA}\log m(a),
\label{eq:lam-m}\\
\Lambda_{\mathrm{task}}
&\;:=\;
\sup_{u'\in\supp p(\cdot\mid s)}
\Bigl(\max_{a\in\cA}\log\mu(a\mid s,u')-\min_{a\in\cA}\log\mu(a\mid s,u')\Bigr)
\;\in\;[0,+\infty],
\label{eq:lam-task}
\end{align}
where the supremum runs over tasks $u'$ under which $\mu(\cdot\mid s,u')$ is not identically zero on $\cA$, with the convention $\log 0=-\infty$.
The spreads $\Lambda_\mu$ and $\Lambda_m$ are always finite, and $\Lambda_{\mathrm{task}}$ is finite in particular whenever no candidate is ruled out by any task in the support.
%Of the three, $\Lambda_\mu$ and $\Lambda_{\mathrm{task}}$ are properties of the faithful model alone; $\Lambda_m$ depends on the corruption only through the task weights $\omega$.

\begin{theorem}[Reverse selection is bounded by the faithful spread; forward selection by the corruption]
\label{thm:bias}
Fix a state $s$, a task $u$ in the support of $p(\cdot\mid s)$, and a finite candidate set $\cA\subseteq\supp\mu(\cdot\mid s,u)$ with $|\cA|\ge2$.
Let Assumptions~\ref{ass:corruption} and \ref{ass:coherence} hold and let $\beta>0$.
Let $a^{+}:=\argmax_{a\in\cA}\varepsilon(s,a)$ be the most inflated candidate, $\gamma:=\mu_{\cA}(a^{+})$ its faithful mass, and
\begin{equation}
\label{eq:delta2}
\Delta_2
\;:=\;
\varepsilon(s,a^{+})-\max_{a\in\cA\setminus\{a^{+}\}}\varepsilon(s,a)
\;\ge\;0
\end{equation}
the inflation gap between the most inflated candidate and the rest.
Then:
\begin{enumerate}[leftmargin=1.5em, topsep=0pt, itemsep=0pt]
\item[(i)] \textbf{The reverse bias is bounded independently of the corruption.}
\begin{equation}
\label{eq:revbound}
\dtv(\rho,\mu_{\cA})
\;\le\;
\tanh\!\Bigl(\frac{|\beta-1|\,\Lambda_\mu+\beta\Lambda_m}{4}\Bigr)
\;\le\;
\tanh\!\Bigl(\frac{|\beta-1|\,\Lambda_\mu+\beta\Lambda_{\mathrm{task}}}{4}\Bigr),
\end{equation}
and the right-most bound does not involve $\varepsilon$ at all; at $\beta=1$ it reads $\dtv(\rho,\mu_{\cA})\le\tanh(\Lambda_m/4)\le\tanh(\Lambda_{\mathrm{task}}/4)$.
\item[(ii)] \textbf{The forward bias grows with the corruption.}
\begin{equation}
\label{eq:fwdbound}
\dtv(\mu_{\cA},\pi_{\cA})
\;\ge\;
\frac{\gamma(1-\gamma)\bigl(1-e^{-\Delta_2}\bigr)}
{\gamma+(1-\gamma)e^{-\Delta_2}},
\end{equation}
with equality when $|\cA|=2$; the right-hand side is increasing in $\Delta_2$ and tends to $1-\gamma$ as $\Delta_2\to\infty$.
\end{enumerate}
\end{theorem}

Part~(i) controls the residual bias of the corrected policy by the spread of the \emph{faithful} log-probabilities over the candidate set: at $\beta=1$ the bound is $\tanh(\Lambda_{\mathrm{task}}/4)$, a quantity of the uncorrupted model, whatever the size of $\varepsilon$.
Part~(ii) shows that the forward-selection bias is a function of the corruption itself: it grows with the inflation gap $\Delta_2$ between the most inflated candidate and the rest, toward $1-\gamma$.
Where the corruption is flat, forward selection is exact and the corrected policy retains only the $m$-bias; the retry loop is the regime in which the two are farthest apart, with forward selection's bias growing while the corrected policy's stays capped by a corruption-free quantity.

\section{Experiments}
\label{sec:experiments}
\subsection{Experimental setup}
\label{sec:experimental_setup}

\paragraph{Benchmarks.}
We evaluate our proposed method on four multi-turn agent benchmarks
spanning diverse task domains: ALFWorld~\citep{Shridhar2020ALFWorldAT}
for household task execution,
ScienceWorld~\citep{Wang2022ScienceWorldIY} for interactive scientific
reasoning, BabyAI~\citep{ChevalierBoisvert2018BabyAIAP} for instruction
following in gridworld environments, and
Jericho~\citep{Hausknecht2019InteractiveFG} for text-based adventure
games.

\paragraph{Backbone models and inference.}
We use LLaDA-8B-Instruct~\citep{Nie2025LargeLD} and
iLLaDA-8B-Instruct~\citep{Nie2026ImprovedLL} as the backbone models
for our dLLM-based agents. iLLaDA follows LLaDA's masked diffusion
formulation and is trained from scratch with larger-scale
pre-training and supervised fine-tuning. For efficient inference,
we use the Fast-dLLM~\citep{Wu2025FastdLLMTA} implementation for
LLaDA and adapt its block-wise approximate key--value (KV) caching
and parallel decoding to iLLaDA.

\paragraph{Baseline.}
We compare our proposed method, \textbf{Reflect Reverse}, with \textbf{One Pass}.
Following the action selection procedure of \citet{Lu2026TheBL},
this baseline executes the first valid action extracted from
the generated output.

\paragraph{Evaluation protocol.}
We evaluate both methods at two sampling temperatures
$T \in \{0.8, 0.9\}$ for LLaDA and $T \in \{0.9, 1.0\}$ for iLLaDA, 
with three random seeds 12, 22 and 32. Each evaluation episode is limited to
30 interaction steps and terminates early if the task is completed.
At each step, we generate 20 candidate actions and apply the
corresponding action-selection method. Generated actions are mapped
to valid environment actions through similarity-based matching,
with a match accepted only when its similarity score exceeds 0.5. Complete hyperparameters for the experiments are summarized in Appendix \ref{app:hyperparameters}.

\paragraph{Evaluation metrics.}
For each benchmark, we report success rate (SR) and progress rate
(PR) \citep{ma2024agentboard} to capture complementary aspects of agent performance.
Success rate measures the fraction of evaluation episodes in which
the agent completes the task within the interaction budget.
Progress rate measures advancement toward the task objective,
averaged across evaluation episodes.

\subsection{Main results}
Figure~\ref{fig:LLaDA-iLLaDA-score} compares \textbf{Reflect Reverse} with \textbf{One Pass} across the four benchmarks and two backbone models. \textbf{Reflect Reverse} achieves higher mean SR in 13 of the 16 evaluation groups and higher mean PR in 14. In particular, SR improves on ALFWorld, ScienceWorld, and BabyAI for both backbones at both evaluated temperatures, while PR improves in all eight configurations for iLLaDA.

Notably, SR improves on ALFWorld, ScienceWorld, and BabyAI for both backbones at both evaluated temperatures.
For iLLaDA, PR also improves across all four benchmarks at both temperatures, demonstrating consistent gains in
advancement toward task objectives.

The improvements are substantial in several settings. For LLaDA, the largest absolute SR gain occurs on ALFWorld at $T=0.9$, where SR increases from 8.71\% to 16.17\%, accompanied by an increase in PR from 29.33\% to 36.24\%. For iLLaDA, the largest absolute SR gain occurs on BabyAI at $T=1.0$, where SR rises from 11.61\% to 20.24\% and PR increases from 25.13\% to 35.40\%.

The benefits also extend to partial progress on Jericho, where PR improves for both backbones at both temperatures, although SR decreases at the lower temperature evaluated for each backbone. Across all configurations, the only PR decreases occur in two LLaDA settings and are smaller than
one percentage point in both cases. Taken together, these results support reverse reflection
as an effective approach to improving multi-turn agent performance, with gains spanning both backbone models
and diverse task domains.
\begin{figure*}[h]
  \centering

  \begin{subfigure}{\linewidth}
    \centering
    \includegraphics[width=\linewidth]{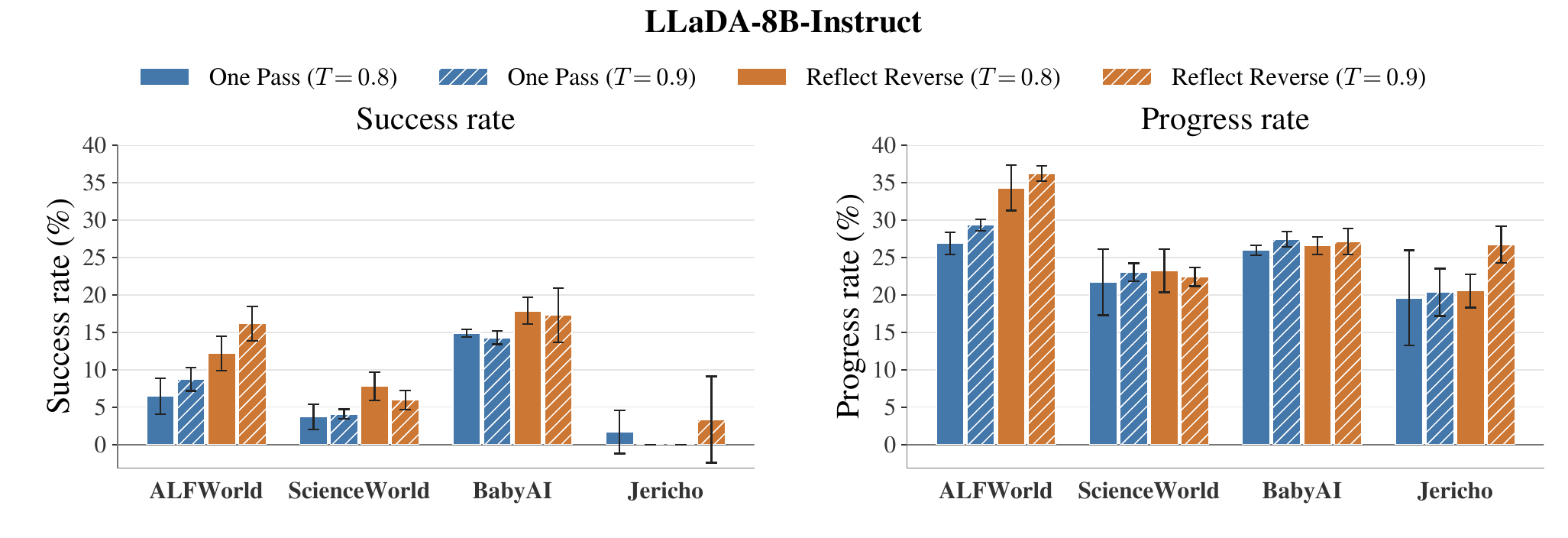}
    % \caption{LLaDA}
    \label{fig:LLaDA-score}
  \end{subfigure}
  \vspace{-2.0em}
  \begin{subfigure}{\linewidth}
    \centering
    \includegraphics[width=\linewidth]{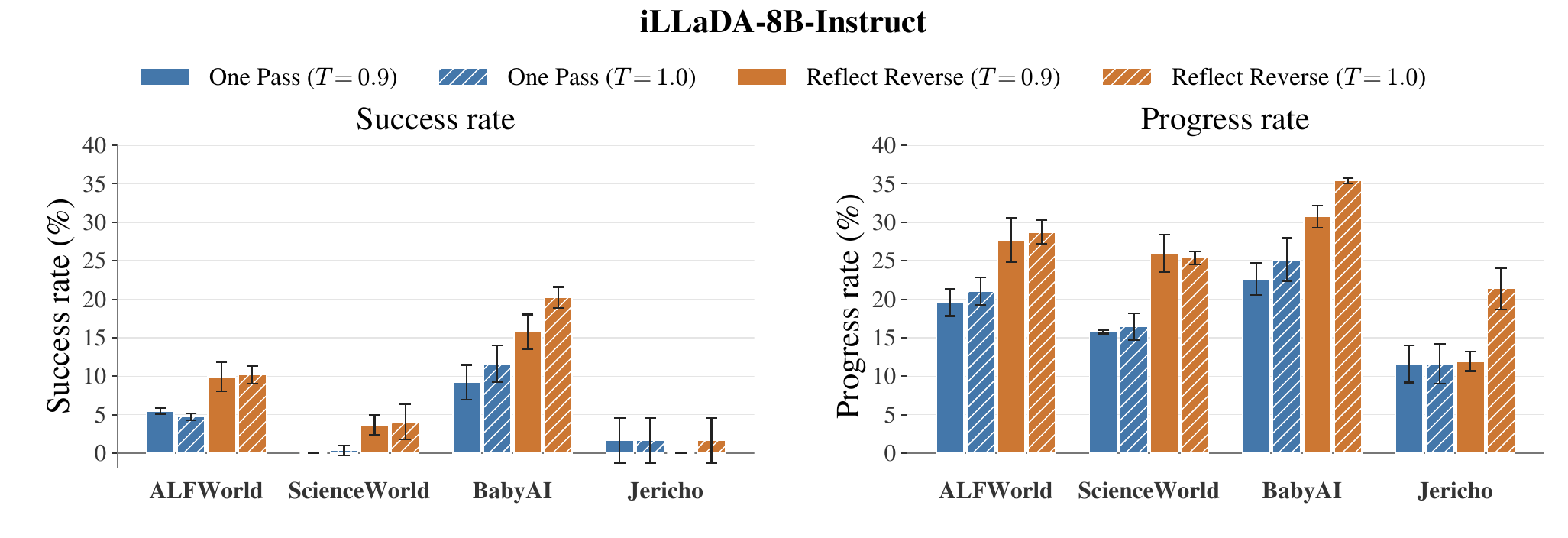}
    % \caption{iLLaDA}
    \label{fig:iLLaDA-score}
  \end{subfigure}
  \vspace{-0.5em}
  \caption {Success rate and progress rate for LLaDA-8B-Instruct and iLLaDA-8B-Instruct.}
  \label{fig:LLaDA-iLLaDA-score}
\end{figure*}

% The second follows the same procedure as \textbf{Reflect Reverse}, but selects actions using weighted scores computed over the candidate actions rather than the task description \textbf{Reflect Forward} \textcolor{red}{add equation and method name here}.

\subsection{Ablations}
\subsubsection{Rerank on forward score}
To evaluate the contribution of reverse conditional scoring, we compare \textbf{Reflect Reverse} with
\textbf{Reflect Forward}. Rather than masking task description tokens $u$ as in~Eq.\ref{eq:score},
\textbf{Reflect Forward} masks candidate action tokens $a$ and scores their reconstruction conditioned on $s$, $u$, and the unmasked action tokens:
\begin{equation}
\label{eq:forward_score}
\psi(a) := \frac{1}{M}\sum_{m=1}^{M}\frac{1}{|B_m|} \sum_{j\in B_m} \log\pist(a_j\mid s,u,\tilde a^{(m)}),
\end{equation}
where $B_m \subseteq \{1,\dots,|a|\}$ is a randomly sampled, nonempty set of masked positions, and $\tilde a^{(m)}$ denotes the candidate action with the tokens at those positions masked. The Act step retains the same softmax action-selection policy, replacing $\phi(a)$ with $\psi(a)$ when sampling the next action.

\begin{table*}[!htp]
\centering
\caption{
Comparison of success rate (SR) and progress rate (PR) for
\textbf{Reflect Reverse} and \textbf{Reflect Forward}, averaged across the two
sampling temperatures. Reverse and Forward results are reported as
mean $\pm$ standard deviation (\%). The relative change is computed as
$\Delta=(\text{Reverse}-\text{Forward})/\text{Forward}\times100\%$.
Positive relative changes are shown in bold. When the Forward value is zero,
the relative change is undefined and reported as N/A.
}
\label{tab:reflect-forward-reverse-mean}

\footnotesize
\setlength{\tabcolsep}{2.5pt} % reduced from 4pt
\renewcommand{\arraystretch}{1.08}

\begin{tabular}{@{}llccc@{\hspace{3pt}}ccc@{}}
\toprule
\textbf{Model}
& \textbf{Benchmark}
& \multicolumn{3}{c}{\textbf{Success Rate (SR)}}
& \multicolumn{3}{c}{\textbf{Progress Rate (PR)}} \\
\cmidrule(lr){3-5}
\cmidrule(lr){6-8}
&
& \textbf{Reverse}
& \textbf{Forward}
& $\boldsymbol{\Delta}$ (\%)
& \textbf{Reverse}
& \textbf{Forward}
& $\boldsymbol{\Delta}$ (\%) \\
\midrule

LLaDA
& ALFWorld
& $14.18 \pm 2.99$
& $8.21 \pm 1.33$
& $\mathbf{+72.72}$
& $35.26 \pm 2.29$
& $30.59 \pm 1.79$
& $\mathbf{+15.27}$ \\
\cmidrule(lr){2-8}

& ScienceWorld
& $6.85 \pm 1.78$
& $4.63 \pm 1.78$
& $\mathbf{+47.95}$
& $22.82 \pm 2.06$
& $21.15 \pm 2.43$
& $\mathbf{+7.90}$ \\
\cmidrule(lr){2-8}

& BabyAI
& $17.56 \pm 2.57$
& $17.56 \pm 3.32$
& $0.00$
& $26.83 \pm 1.35$
& $29.40 \pm 2.35$
& $-8.74$ \\
\cmidrule(lr){2-8}

& Jericho
& $1.67 \pm 4.08$
& $1.67 \pm 2.58$
& $0.00$
& $23.63 \pm 3.99$
& $19.23 \pm 2.61$
& $\mathbf{+22.88}$ \\

\midrule

iLLaDA
& ALFWorld
& $10.07 \pm 1.40$
& $0.75 \pm 0.94$
& $\mathbf{+1242.67}$
& $28.20 \pm 2.13$
& $16.08 \pm 1.56$
& $\mathbf{+75.37}$ \\
\cmidrule(lr){2-8}

& ScienceWorld
& $3.89 \pm 1.69$
& $0.19 \pm 0.45$
& $\mathbf{+1947.37}$
& $25.67 \pm 1.66$
& $8.76 \pm 1.68$
& $\mathbf{+193.04}$ \\
\cmidrule(lr){2-8}

& BabyAI
& $18.01 \pm 2.96$
& $8.48 \pm 1.23$
& $\mathbf{+112.38}$
& $33.07 \pm 2.73$
& $19.98 \pm 2.42$
& $\mathbf{+65.52}$ \\
\cmidrule(lr){2-8}

& Jericho
& $0.83 \pm 2.04$
& $0.00 \pm 0.00$
& \text{N/A}
& $16.65 \pm 5.50$
& $10.57 \pm 1.49$
& $\mathbf{+57.52}$ \\

\bottomrule
\end{tabular}
\end{table*}

Table 1 compares the mean success rate and progress rate across four benchmarks, averaged over two sampling temperatures, and reports the relative change ($\Delta$) between \textbf{Reflect Reverse} and \textbf{Reflect Forward}. Consistent with these averaged results, the detailed results in Tables~\ref{tab:agentboard-results} and~\ref{tab:agentboard-iLLaDA-results} in Appendix~\ref{app:rsupplementary} show that \textbf{Reflect Reverse} generally outperforms \textbf{Reflect Forward}. Across the 16 configurations, reverse scoring achieves higher SR in 12 and matches it at the reported precision in three others. It also achieves higher PR in 14 configurations.

For LLaDA, reverse scoring improves both SR and PR on ALFWorld and ScienceWorld at both $T=0.8$ and $T=0.9$.
On ALFWorld at $T=0.9$, SR increases from 8.96\% to 16.17\% and PR rises from 31.45\% to 36.24\%. On Jericho, reverse scoring matches forward scoring in SR at both temperatures while improving PR from 17.52\% to 20.53\% at $T=0.8$ and from 20.94\% to 26.72\% at $T=0.9$.

The gains are particularly consistent for iLLaDA, where reverse scoring improves SR in seven of eight configurations and PR in all eight. On BabyAI, SR increases from 8.93\% to 15.77\% at $T=0.9$ and from 8.04\% to 20.24\% at $T=1.0$, while PR rises from 21.52\% to 30.73\%
and from 18.44\% to 35.40\%, respectively. For LLaDA on BabyAI, reverse scoring improves SR at $T=0.9$ but is less consistent on PR. Overall, these results support reverse conditional scoring as an effective candidate-selection criterion, with advantages over forward scoring across both backbones and the sampling temperatures evaluated for each.

\subsubsection{Retry State Loops}
The dominant retry-loop patterns consist of repeatedly executing the same action (e.g., \texttt{AAAAAA}) or alternating between two actions (e.g., \texttt{ABABABAB}). Examples of both pattern can be found in Appendix \ref{app:retrystateexample}. We count each additional complete occurrence beyond the initial action or action pair as one retry, without double-counting overlaps. Thus, \texttt{AAAAAA} contains five retries, while \texttt{ABABABAB} contains three. 

The retry percentage is defined as the total number of retries divided by the interaction budget. As shown in Figure~\ref{fig:LLaDA-iLLaDA-retry}, repetitive behavior is strongly associated with sampling temperature: lower temperatures yield higher mean retry percentages in 21 of the 24 pairwise temperature comparisons, suggesting that reduced action diversity increases agent's repetitive behavior. Despite this general trend, \textbf{Reflect Reverse} consistently mitigates repetition, achieving the lowest mean retry percentage in all 16 evaluated settings. By contrast, \textbf{Reflect Forward} increases repetition relative to \textbf{One Pass} in five of the eight iLLaDA settings.

While these results show that \textbf{Reflect Reverse} is effective at mitigating repetitive behavior, lower retry frequency does not necessarily translate directly into better task performance. For example, on ALFWorld at $T=0.9$, iLLaDA with \textbf{Reflect Forward} exhibits fewer retries than \textbf{One Pass}, yet achieves lower success and progress rates. This indicates that avoiding repetitive action loops is only one component of effective agent behavior: successful task execution also depends on the model's ability to generate correct and task-appropriate actions.

\begin{figure*}[htp!]
  \centering

  \begin{subfigure}{\linewidth}
    \centering
    \includegraphics[width=\linewidth]{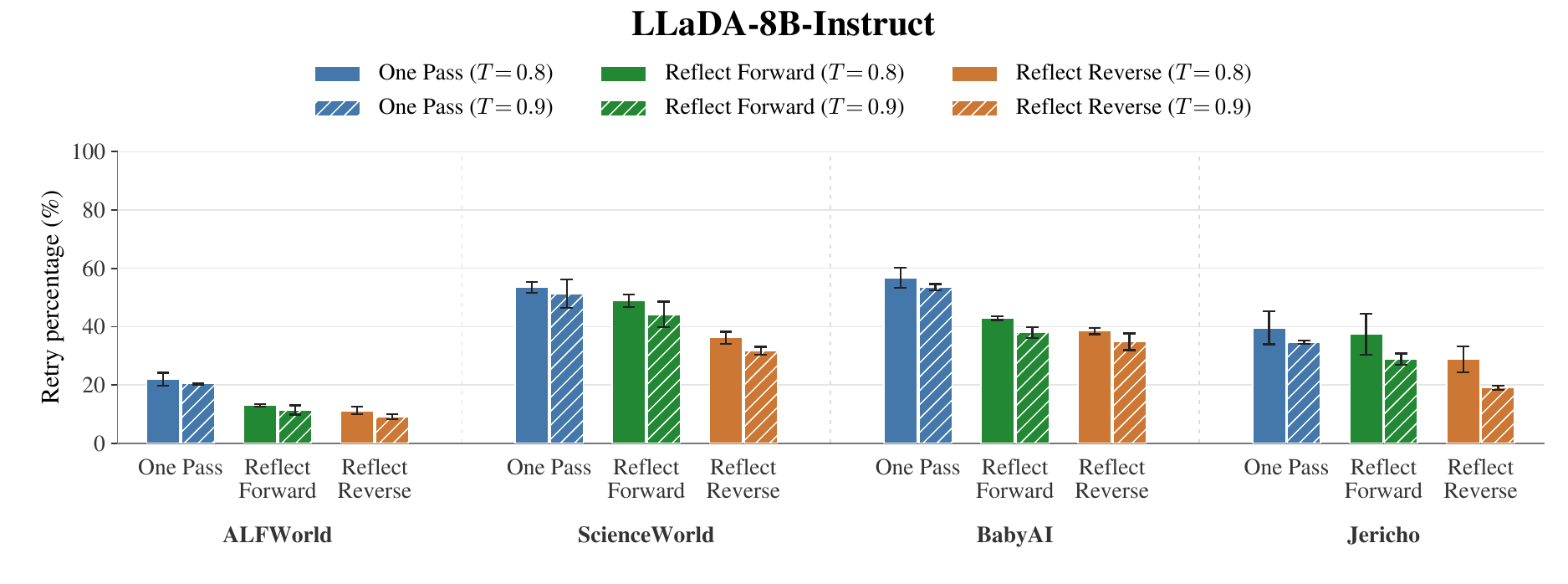}
    % \caption{LLaDA}
    \label{fig:LLaDA-retry}
  \end{subfigure}
  \vspace{-1.5em}

  \begin{subfigure}{\linewidth}
    \centering
    \includegraphics[width=\linewidth]{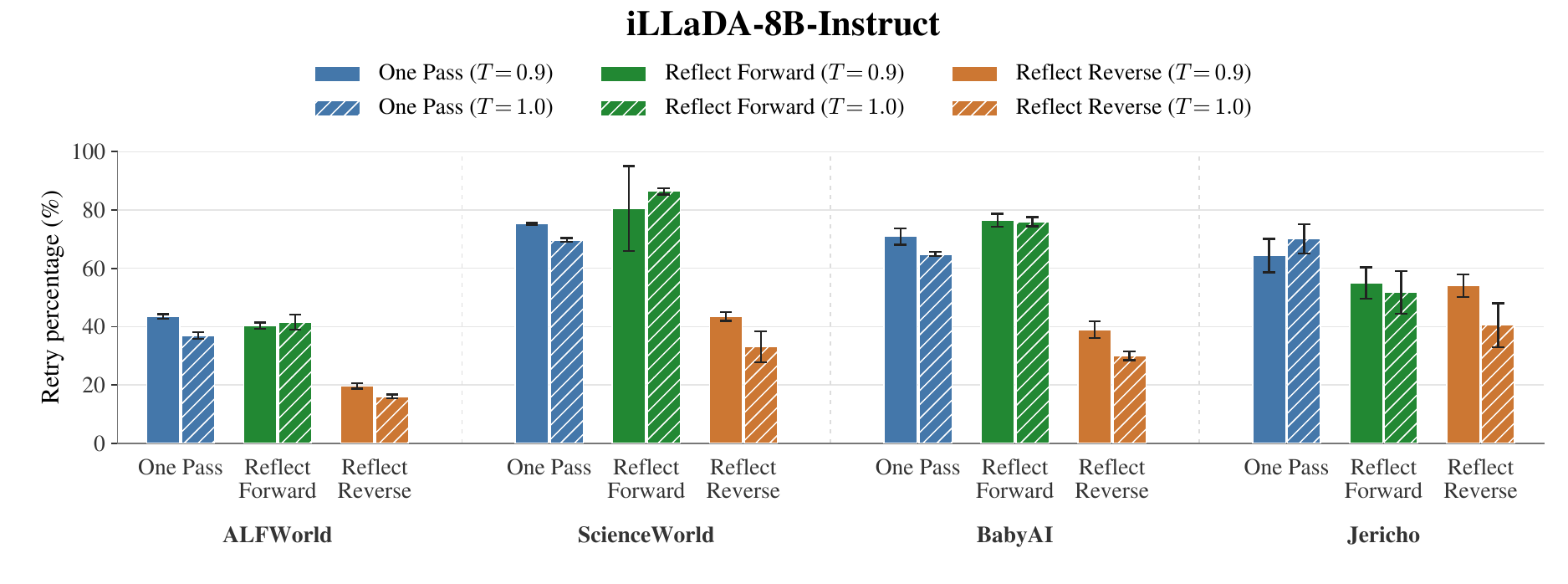}
    % \caption{iLLaDA}
    \label{fig:iLLaDA-retry}
  \end{subfigure}
  \vspace{-3.0em}
  \caption {Retry state percentage of  LLaDA-8B-Instruct and iLLaDA-8B-Instruct.}
  \label{fig:LLaDA-iLLaDA-retry}
\end{figure*}

\section{Related Work}
\label{sec:related}

\paragraph{Diffusion language models and adaptive decoding.}
Masked dLLMs such as LLaDA \citep{Nie2025LargeLD} and iLLaDA \citep{Nie2026ImprovedLL} generate by iteratively denoising masked spans; the standard inference recipe couples confidence-based remasking with semi-autoregressive block decoding \citep{Nie2025LargeLD,Wu2025FastdLLMTA,arriola2025block}.
\citet{ni2026flexibilitytraprethinkingvalue} identify a failure mode of this recipe in reasoning tasks: because the sampler commits the positions it is most confident about and defers the rest, it systematically bypasses the high-entropy ``forking'' tokens at which reasoning paths diverge; by the time these positions are filled, the committed context has already resolved them, so an open decision degenerates into a retrospective alignment with a pre-determined gap (\emph{entropy degradation}).
Forcing left-to-right order restores the confrontation of forks and enlarges the reachable solution space, with the effect growing monotonically in the degree of order freedom.
The retry loops we study are this trap in a sharper form.
At a failure state, \emph{which action now?} is a fork, and unlike in reasoning tasks the context already contains a confident fill for it: the failed action, verbatim in $s$.
Adaptive decoding commits the copy without the open decision ever being sampled, and the deferred positions are finally filled in alignment with the retry.
%Two features of the agentic setting deserve emphasis.
%First, our agent thinks before acting, i.e. the denoised span is a chain of thought followed by the action, but the reasoning does not circumvent the trap: every position, in the reasoning or in the action, can be deferred until the committed context makes it confident, so the thinking never confronts the fork either.
%Second, the mechanism is independent of decoding parallelism: what matters is that the sampler chooses \emph{which} position to commit by confidence, so the trap persists even at one token per step, whereas an autoregressive model has no such choice and must decide the first token of the new action while the failure feedback is the newest evidence.

\paragraph{LLM agents in embodied environments.}
Reasoning-and-acting agent frameworks \citep{yao2023reactsynergizingreasoningacting, shinn2023reflexionlanguageagentsverbal,mowerRobotOperatingSystem2026} couple chain-of-thought traces with environment actions and are standardly evaluated \citep{liu2024agentbench,ma2024agentboard} on ALFWorld \citep{Shridhar2020ALFWorldAT}, ScienceWorld \citep{Wang2022ScienceWorldIY}, BabyAI \citep{ChevalierBoisvert2018BabyAIAP}, and Jericho \citep{Hausknecht2019InteractiveFG}.
Recent evaluations of dLLM-backed agents \citep{Lu2026TheBL} report that the latency gains of parallel decoding do not transfer to agentic competence, with retry loops as a dominant failure mode.
We give this failure a mechanistic account, i.e. adaptive decoding bridging the failure feedback instead of confronting it, and a training-free correction.

\paragraph{Sample-then-rerank action selection.}
Reflect Reverse belongs to the family of decoding-time corrections: best-of-$N$ selection with verifiers \citep{lightman2024let,ji2025surelysafealignmentlarge,kang2026scalable}, self-consistency \citep{wang2023selfconsistencyimproveschainthought}, and PMI-style rescoring of raw likelihoods \citep{holtzman2021surface,ji2026scalablepowersamplingunlocking,nguyen2026modelknowsdecoderfinds}.
These methods share our observation that raw model probability is an unreliable score, but they correct it with an external signal: a learned verifier, a majority vote, or a domain-conditional baseline.
Our score instead is the model's own reverse conditional $\pist(u\mid s,a)$.
The same reverse-conditional identity appears, in a different role, in \citet{lian2026langforcebayesiandecompositionvision}, who train vision--language--action policies by maximizing the log-likelihood ratio $\log p(\ell\mid a,v)-\log p(\ell\mid v)$, the conditional PMI between action and instruction, as a training objective; our correction instead is a training-free, decoding-time rule.

\section{Conclusion And Future Works}

We traced the retry loops that stall diffusion language model agents to a structural property of masked decoding: because the sampler commits the positions it is most confident about and defers the rest, the open decision at a failure state is postponed until the committed context resolves it; and the context already offers a confident fill, the failed action itself.
We formalized the resulting distortion of the action distribution as a \emph{task-blind corruption}, i.e. a state--action factor that inflates contextually salient actions independently of the task, and verified it through goal-swap interventions.
The formalization is what motivates the correction: a factor that carries no task information cannot survive a posterior over tasks, so it cancels from the reverse conditional $\pist(u\mid s,a)$, which masked dLLMs evaluate natively by masking the task tokens and denoising. 
The resulting rule, \textbf{Reflect Reverse}, is training-free, requires no learned critic, value function, or environment rollouts, adds only a few parallel denoising passes per candidate, and improves success and progress rates over forward-scoring baselines on four embodied benchmarks and two backbones.

\paragraph{Future work.}
Firstly, Proposition~\ref{thm:main} governs the comparison among the candidates on the table but proposals are still drawn from the corrupted forward conditional, so the faithful model's support may go uncovered. Extending the invariance to the proposal step (e.g. dividing out an estimated $\varepsilon(s,a)$) would close this gap.
Secondly, the corruption is ultimately a training artifact of adaptive decoding; decoding-time correction could be complemented by fine-tuning at failure states or by regularizing forward--reverse consistency, removing the pathology at its source.
Finally, the retry loop is the agentic form of entropy degradation, and the reverse question, \emph{does this continuation still explain the prompt?}, applies to reasoning chains, suggesting reverse scoring as a general decoding-time instrument for masked diffusion models.

\clearpage

%\subsubsection*{Author Contributions}
%If you'd like to, you may include  a section for author contributions as is done
%in many journals. This is optional and at the discretion of the authors.
%
%\subsubsection*{Acknowledgments}
%Use unnumbered third level headings for the acknowledgments. All
%acknowledgments, including those to funding agencies, go at the end of the paper.

\bibliography{iclr2027_conference}

@misc{ni2026flexibilitytraprethinkingvalue,
      title={The Flexibility Trap: Rethinking the Value of Arbitrary Order in Diffusion Language Models}, 
      author={Zanlin Ni and Shenzhi Wang and Yang Yue and Tianyu Yu and Weilin Zhao and Yeguo Hua and Tianyi Chen and Jun Song and Cheng Yu and Bo Zheng and Gao Huang},
      year={2026},
      eprint={2601.15165},
      archivePrefix={arXiv},
      primaryClass={cs.CL},
      url={https://arxiv.org/abs/2601.15165}, 
}

@article{mowerRobotOperatingSystem2026,
  title = {A Robot Operating System Framework for Using Large Language Models in Embodied {{AI}}},
  author = {Mower, Christopher E. and Wan, Yuhui and Yu, Hongzhan and Grosnit, Antoine and Gonzalez-Billandon, Jonas and Zimmer, Matthieu and Liu, Puze and Palenicek, Daniel and Tateo, Davide and Peters, Jan and Qu, Kaixian and Zhang, Mike and Lan, Guowei and Cramariuc, Andrei and Cadena, Cesar and Hutter, Marco and Tian, Guangjian and Zhuang, Yuzhen and Shao, Kun and Quan, Xingyue and Hao, Jianye and Wang, Jun and Bou-Ammar, Haitham},
  year = 2026,
  journal={Nature Machine Intelligence},
  publisher={Nature Publishing Group},
  volume = {8},
  number = {3},
  pages = {313--325},
  issn = {2522-5839},
  doi = {10.1038/s42256-026-01186-z},
  url = {https://doi.org/10.1038/s42256-026-01186-z},
}

@misc{lian2026langforcebayesiandecompositionvision,
      title={LangForce: Bayesian Decomposition of Vision Language Action Models via Latent Action Queries}, 
      author={Shijie Lian and Bin Yu and Xiaopeng Lin and Laurence T. Yang and Zhaolong Shen and Changti Wu and Yuzhuo Miao and Cong Huang and Kai Chen},
      year={2026},
      eprint={2601.15197},
      archivePrefix={arXiv},
      primaryClass={cs.AI},
      url={https://arxiv.org/abs/2601.15197}, 
}

@article{Binette_2019,
   title={A Note on Reverse Pinsker Inequalities},
   volume={65},
   ISSN={1557-9654},
   url={http://dx.doi.org/10.1109/TIT.2019.2896192},
   DOI={10.1109/tit.2019.2896192},
   number={7},
   journal={IEEE Transactions on Information Theory},
   publisher={Institute of Electrical and Electronics Engineers (IEEE)},
   author={Binette, Olivier},
   year={2019},
   month=July, pages={4094–4096} }

@inproceedings{chae2024mitigating,
  title={Mitigating hallucination in abstractive summarization with domain-conditional mutual information},
  author={Chae, Kyubyung and Choi, Jaepill and Jo, Yohan and Kim, Taesup},
  booktitle={Findings of the association for computational linguistics: NAACL 2024},
  pages={1809--1820},
  year={2024}
}

@misc{yao2023reactsynergizingreasoningacting,
      title={ReAct: Synergizing Reasoning and Acting in Language Models}, 
      author={Shunyu Yao and Jeffrey Zhao and Dian Yu and Nan Du and Izhak Shafran and Karthik Narasimhan and Yuan Cao},
      year={2023},
      eprint={2210.03629},
      archivePrefix={arXiv},
      primaryClass={cs.CL},
      url={https://arxiv.org/abs/2210.03629}, 
}

@misc{shinn2023reflexionlanguageagentsverbal,
      title={Reflexion: Language Agents with Verbal Reinforcement Learning}, 
      author={Noah Shinn and Federico Cassano and Edward Berman and Ashwin Gopinath and Karthik Narasimhan and Shunyu Yao},
      year={2023},
      eprint={2303.11366},
      archivePrefix={arXiv},
      primaryClass={cs.AI},
      url={https://arxiv.org/abs/2303.11366}, 
}

@misc{wang2023selfconsistencyimproveschainthought,
      title={Self-Consistency Improves Chain of Thought Reasoning in Language Models}, 
      author={Xuezhi Wang and Jason Wei and Dale Schuurmans and Quoc Le and Ed Chi and Sharan Narang and Aakanksha Chowdhery and Denny Zhou},
      year={2023},
      eprint={2203.11171},
      archivePrefix={arXiv},
      primaryClass={cs.CL},
      url={https://arxiv.org/abs/2203.11171}, 
}

@misc{ji2025surelysafealignmentlarge,
      title={On Almost Surely Safe Alignment of Large Language Models at Inference-Time}, 
      author={Xiaotong Ji and Shyam Sundhar Ramesh and Matthieu Zimmer and Ilija Bogunovic and Jun Wang and Haitham Bou Ammar},
      year={2025},
      eprint={2502.01208},
      archivePrefix={arXiv},
      primaryClass={cs.LG},
      url={https://arxiv.org/abs/2502.01208}, 
}

@article{kang2026scalable,
  title={Scalable best-of-n selection for large language models via self-certainty},
  author={Kang, Zhewei and Zhao, Xuandong and Song, Dawn},
  journal={Advances in neural information processing systems},
  volume={38},
  pages={19720--19745},
  year={2026}
}

@inproceedings{holtzman2021surface,
  title={Surface form competition: Why the highest probability answer isn’t always right},
  author={Holtzman, Ari and West, Peter and Shwartz, Vered and Choi, Yejin and Zettlemoyer, Luke},
  booktitle={Proceedings of the 2021 conference on empirical methods in natural language processing},
  pages={7038--7051},
  year={2021}
}

@misc{li2022diffusionlmimprovescontrollabletext,
      title={Diffusion-LM Improves Controllable Text Generation}, 
      author={Xiang Lisa Li and John Thickstun and Ishaan Gulrajani and Percy Liang and Tatsunori B. Hashimoto},
      year={2022},
      eprint={2205.14217},
      archivePrefix={arXiv},
      primaryClass={cs.CL},
      url={https://arxiv.org/abs/2205.14217}, 
}

@article{lou2023discrete,
  title={Discrete diffusion modeling by estimating the ratios of the data distribution},
  author={Lou, Aaron and Meng, Chenlin and Ermon, Stefano},
  journal={arXiv preprint arXiv:2310.16834},
  year={2023}
}

@inproceedings{gong2025scaling,
  title={Scaling diffusion language models via adaptation from autoregressive models},
  author={Gong, Shansan and Agarwal, Shivam and Zhang, Yizhe and Ye, Jiacheng and Zheng, Lin and Li, Mukai and An, Chenxin and Zhao, Peilin and Bi, Wei and Han, Jiawei and others},
  booktitle={International Conference on Learning Representations},
  volume={2025},
  pages={5046--5073},
  year={2025}
}

@article{holtzman2019curious,
  title={The curious case of neural text degeneration},
  author={Holtzman, Ari and Buys, Jan and Du, Li and Forbes, Maxwell and Choi, Yejin},
  journal={arXiv preprint arXiv:1904.09751},
  year={2019}
}

@article{austin2021structured,
  title={Structured denoising diffusion models in discrete state-spaces},
  author={Austin, Jacob and Johnson, Daniel D and Ho, Jonathan and Tarlow, Daniel and Van Den Berg, Rianne},
  journal={Advances in neural information processing systems},
  volume={34},
  pages={17981--17993},
  year={2021}
}

@article{sahoo2024simple,
  title={Simple and effective masked diffusion language models},
  author={Sahoo, Subham S and Arriola, Marianne and Schiff, Yair and Gokaslan, Aaron and Marroquin, Edgar and Chiu, Justin T and Rush, Alexander and Kuleshov, Volodymyr},
  journal={Advances in Neural Information Processing Systems},
  volume={37},
  pages={130136--130184},
  year={2024}
}

@article{shi2024simplified,
  title={Simplified and generalized masked diffusion for discrete data},
  author={Shi, Jiaxin and Han, Kehang and Wang, Zhe and Doucet, Arnaud and Titsias, Michalis},
  journal={Advances in neural information processing systems},
  volume={37},
  pages={103131--103167},
  year={2024}
}

@inproceedings{schiff2025simple,
  title={Simple guidance mechanisms for discrete diffusion models},
  author={Schiff, Yair and Sahoo, Subham and Phung, Hao and Wang, Guanghan and Boshar, Sam and Dalla-Torre, Hugo and Almeida, Bernardo and Rush, Alexander and Pierrot, Thomas and Kuleshov, Volodymyr},
  booktitle={International Conference on Learning Representations},
  volume={2025},
  pages={43776--43821},
  year={2025}
}

@article{wei2022chain,
  title={Chain-of-thought prompting elicits reasoning in large language models},
  author={Wei, Jason and Wang, Xuezhi and Schuurmans, Dale and Bosma, Maarten and Xia, Fei and Chi, Ed and Le, Quoc V and Zhou, Denny and others},
  journal={Advances in neural information processing systems},
  volume={35},
  pages={24824--24837},
  year={2022}
}

@article{gulrajani2023likelihood,
  title={Likelihood-based diffusion language models},
  author={Gulrajani, Ishaan and Hashimoto, Tatsunori B},
  journal={Advances in Neural Information Processing Systems},
  volume={36},
  pages={16693--16715},
  year={2023}
}

@inproceedings{ou2025your,
  title={Your absorbing discrete diffusion secretly models the conditional distributions of clean data},
  author={Ou, Jingyang and Nie, Shen and Xue, Kaiwen and Zhu, Fengqi and Sun, Jiacheng and Li, Zhenguo and Li, Chongxuan},
  booktitle={International Conference on Learning Representations},
  volume={2025},
  pages={64972--65009},
  year={2025}
}

@article{ma2024agentboard,
  title={Agentboard: An analytical evaluation board of multi-turn llm agents},
  author={Ma, Chang and Zhang, Junlei and Zhu, Zhihao and Yang, Cheng and Yang, Yujiu and Jin, Yaohui and Lan, Zhenzhong and Kong, Lingpeng and He, Junxian},
  journal={Advances in neural information processing systems},
  volume={37},
  pages={74325--74362},
  year={2024}
}

@inproceedings{arriola2025block,
  title={Block diffusion: Interpolating between autoregressive and diffusion language models},
  author={Arriola, Marianne and Gokaslan, Aaron and Chiu, Justin and Yang, Zhihan and Qi, Zhixuan and Han, Jiaqi and Sahoo, Subham and Kuleshov, Volodymyr},
  booktitle={International Conference on Learning Representations},
  volume={2025},
  pages={50726--50753},
  year={2025}
}

@inproceedings{liu2024agentbench,
  title={Agentbench: Evaluating llms as agents},
  author={Liu, Xiao and Yu, Hao and Zhang, Hanchen and Xu, Yifan and Lei, Xuanyu and Lai, Hanyu and Gu, Yu and Ding, Hangliang and Men, Kaiwen and Yang, Kejuan and others},
  booktitle={International Conference on Learning Representations},
  volume={2024},
  pages={52989--53046},
  year={2024}
}

@inproceedings{lightman2024let,
  title={Let's verify step by step},
  author={Lightman, Hunter and Kosaraju, Vineet and Burda, Yuri and Edwards, Harrison and Baker, Bowen and Lee, Teddy and Leike, Jan and Schulman, John and Sutskever, Ilya and Cobbe, Karl},
  booktitle={International Conference on Learning Representations},
  volume={2024},
  pages={39578--39601},
  year={2024}
}

@article{Shridhar2020ALFWorldAT,
  title={ALFWorld: Aligning Text and Embodied Environments for Interactive Learning},
  author={Mohit Shridhar and Xingdi Yuan and Marc-Alexandre C{\^o}t{\'e} and Yonatan Bisk and Adam Trischler and Matthew J. Hausknecht},
  journal={ArXiv},
  year={2020},
  volume={abs/2010.03768},
  url={https://api.semanticscholar.org/CorpusID:222208810}
}

@inproceedings{Wang2022ScienceWorldIY,
  title={ScienceWorld: Is your Agent Smarter than a 5th Grader?},
  author={Ruoyao Wang and Peter Alexander Jansen and Marc-Alexandre C{\^o}t{\'e} and Prithviraj Ammanabrolu},
  booktitle={Conference on Empirical Methods in Natural Language Processing},
  year={2022},
  url={https://api.semanticscholar.org/CorpusID:247451124}
}

@inproceedings{ChevalierBoisvert2018BabyAIAP,
  title={BabyAI: A Platform to Study the Sample Efficiency of Grounded Language Learning},
  author={Maxime Chevalier-Boisvert and Dzmitry Bahdanau and Salem Lahlou and Lucas Willems and Chitwan Saharia and Thien Huu Nguyen and Yoshua Bengio},
  booktitle={International Conference on Learning Representations},
  year={2018},
  url={https://api.semanticscholar.org/CorpusID:59536625}
}

@inproceedings{Hausknecht2019InteractiveFG,
  title={Interactive Fiction Games: A Colossal Adventure},
  author={Matthew J. Hausknecht and Prithviraj Ammanabrolu and Marc-Alexandre C{\^o}t{\'e} and Xingdi Yuan},
  booktitle={AAAI Conference on Artificial Intelligence},
  year={2019},
  url={https://api.semanticscholar.org/CorpusID:202565447}
}

@article{Nie2026ImprovedLL,
  title={Improved Large Language Diffusion Models},
  author={Shen Nie and Qi-Yang Min and Shaoxuan Xu and Zihao Huang and Yuxuan Song and Shan Yong and Yankai Lin and Wayne Xin Zhao and Chongxuan Li and Jiaxin Wen},
  journal={ArXiv},
  year={2026},
  volume={abs/2606.25331},
  url={https://api.semanticscholar.org/CorpusID:289630838}
}

@article{Nie2025LargeLD,
  title={Large Language Diffusion Models},
  author={Shen Nie and Fengqi Zhu and Zebin You and Xiaolu Zhang and Jingyang Ou and Jun Hu and Jun Zhou and Yankai Lin and Jirong Wen and Chongxuan Li},
  journal={ArXiv},
  year={2025},
  volume={abs/2502.09992},
  url={https://api.semanticscholar.org/CorpusID:276395038}
}

@article{Wu2025FastdLLMTA,
  title={Fast-dLLM: Training-free Acceleration of Diffusion LLM by Enabling KV Cache and Parallel Decoding},
  author={Chengyue Wu and Hao Zhang and Shuchen Xue and Zhijian Liu and Shizhe Diao and Ligeng Zhu and Ping Luo and Song Han and Enze Xie},
  journal={ArXiv},
  year={2025},
  volume={abs/2505.22618},
  url={https://api.semanticscholar.org/CorpusID:278959508}
}

@article{Lu2026TheBL,
  title={The Bitter Lesson of Diffusion Language Models for Agentic Workflows: A Comprehensive Reality Check},
  author={Qingyu Lu and Liang Ding and Kanjian Zhang and Jinxia Zhang and Dacheng Tao},
  journal={ArXiv},
  year={2026},
  volume={abs/2601.12979},
  url={https://api.semanticscholar.org/CorpusID:284910491}
}

@misc{nguyen2026modelknowsdecoderfinds,
      title={The Model Knows, the Decoder Finds: Future Value Guided Particle Power Sampling}, 
      author={Tu Nguyen and Matthieu Zimmer and Rasul Tutunov and Xiaotong Ji and Haitham Bou Ammar},
      year={2026},
      eprint={2605.02427},
      archivePrefix={arXiv},
      primaryClass={cs.AI},
      url={https://arxiv.org/abs/2605.02427}, 
}

@misc{ji2026scalablepowersamplingunlocking,
      title={Scalable Power Sampling: Unlocking Efficient, Training-Free Reasoning for LLMs via Distribution Sharpening}, 
      author={Xiaotong Ji and Rasul Tutunov and Matthieu Zimmer and Haitham Bou Ammar},
      year={2026},
      eprint={2601.21590},
      archivePrefix={arXiv},
      primaryClass={cs.LG},
      url={https://arxiv.org/abs/2601.21590}, 
}

@misc{christianos2023panguagentfinetunablegeneralistagent,
      title={Pangu-Agent: A Fine-Tunable Generalist Agent with Structured Reasoning}, 
      author={Filippos Christianos and Georgios Papoudakis and Matthieu Zimmer and Thomas Coste and Zhihao Wu and Jingxuan Chen and Khyati Khandelwal and James Doran and Xidong Feng and Jiacheng Liu and Zheng Xiong and Yicheng Luo and Jianye Hao and Kun Shao and Haitham Bou-Ammar and Jun Wang},
      year={2023},
      eprint={2312.14878},
      archivePrefix={arXiv},
      primaryClass={cs.AI},
      url={https://arxiv.org/abs/2312.14878}, 
}
\bibliographystyle{iclr2027_conference}

\clearpage
\appendix
\section{Task corruption verification}
\label{app:taskcorruption}
We evaluate task sensitivity on 208 retry states from ALFWorld benchmark. For each state $s_i$, we retain the interaction history and compare the repeated action $a_i^{\mathrm{rep}}$ with a distinct admissible alternative $a_i^{\mathrm{alt}}$, selected by Qwen2.5-7B-Instruct. Both actions remain fixed while the original task $u_{i,0}$ is replaced by $K=10$ randomly sampled task descriptions. Let $p_m(a \mid s,u)$ denote the action score for model $m$: summed autoregressive token log likelihood for Qwen, or a Monte Carlo diffusion variational score estimate for LLaDA. LLaDA uses identical sampled masks across task conditions.

For visualization, both actions are centered on the repeated action's
original-task score:
\begin{equation}
    D_{i,j}^{m,a}
    =
    p_m(a \mid s_i,u_{i,j})
    -
    p_m(a_i^{\mathrm{rep}} \mid s_i,u_{i,0}), \qquad a \in \{a^{\mathrm{rep}}, a^{\mathrm{alt}}\}.
    \label{eq:corruption_centered_score}
\end{equation}
The plotted curves are the state averages,
$\bar{D}_j^{m,a} = N^{-1}\sum_{i=1}^{N} D_{i,j}^{m,a}$.
This common baseline preserves the relative preference between the two
actions; the alternative curve is therefore not centered on its own
original-task score.

We quantify task sensitivity using the action-preference margin:
\begin{equation}
    \Delta_{i,j}^{m}
    =
    p_m(a_i^{\mathrm{alt}} \mid s_i,u_{i,j})
    -
    p_m(a_i^{\mathrm{rep}} \mid s_i,u_{i,j}).
    \label{eq:corruption_preference_margin}
\end{equation}
For each state, we compute its population standard deviation across the
replacement tasks, then average over states:
\begin{equation}
    \sigma_m
    =
    \frac{1}{N}\sum_{i=1}^{N}
    \sqrt{
        \frac{1}{K}\sum_{j=1}^{K}
        \left(\Delta_{i,j}^{m}-\bar{\Delta}_i^{m}\right)^2
    },
    \qquad
    \bar{\Delta}_i^{m}
    =
    \frac{1}{K}\sum_{j=1}^{K}\Delta_{i,j}^{m}.
    \label{eq:corruption_goal_sensitivity}
\end{equation}
The original task is excluded from this standard deviation. Smaller
$\sigma_m$ indicates less variation in action preference across
replacement goals.

We additionally measure how frequently task replacement reverses the
preferred action. Defining
$b_{i,j}^{m}=\mathbf{1}[\Delta_{i,j}^{m}>0]$,
the preference flip rate is
\begin{equation}
    F_m
    =
    \frac{1}{NK}
    \sum_{i=1}^{N}\sum_{j=1}^{K}
    \left|b_{i,j}^{m}-b_{i,0}^{m}\right|.
    \label{eq:corruption_flip_rate}
\end{equation}
Each replacement is compared with the original task, yielding 2,080
comparisons per model. These calculations give
$\sigma_{\mathrm{LLaDA}}=1.01$ and $\sigma_{\mathrm{Qwen}}=4.71$ nats,
with flip rates of $6.6\%$ and $44.4\%$, respectively.

The shaded bands use 2,000 episode-level bootstrap resamples to obtain
pointwise 95\% intervals for each action's change from its own
original-task score. For the alternative curve, these intervals are
shifted by the observed mean original-task preference margin to match
the plotted baseline; they therefore represent uncertainty in the
task-induced change, conditional on that baseline offset.

\clearpage
\section{Reflect reverse algorithm}
\label{app:pseudo}
We summarize the method introduced in Section~\ref{sec:method-algo} in the following pseudocode.
\begin{algorithm}[!htbp]
\caption{Reflect Reverse}
\label{alg:reflect-reverse}
\begin{algorithmic}[1]
\REQUIRE State $s$, task $u$, dLLM $\pist$, number of proposals $K$,
         number of reflection rollouts $M$, inverse temperature $\beta$
\ENSURE Selected action $a^\star$

\STATE \textbf{Propose:}
\STATE Sample $K$ action candidates
       $a_1,\ldots,a_K \sim \pist(\cdot \mid s,u)$
       with ReAct
\STATE Deduplicate candidates to obtain $\mathcal{A}(s)$

\STATE \textbf{Reflect:}
\FOR{each $a \in \mathcal{A}(s)$}
    \STATE $\phi(a) \leftarrow 0$
    \FOR{$m = 1,\ldots,M$}
        \STATE Sample a subset of task-token positions
               $B_m \subseteq \{1,\ldots,|u|\}$
        \STATE Construct $\tilde{u}^{(m)}$ by masking
               positions $B_m$ in $u$
        \STATE Condition the dLLM on $(s,a,\tilde{u}^{(m)})$
        \STATE $\displaystyle
               \ell_m(a) \leftarrow
               \frac{1}{|B_m|}
               \sum_{j \in B_m}
               \log \pist
               \bigl(u_j \mid s,a,\tilde{u}^{(m)}\bigr)$
        \STATE $\phi(a) \leftarrow \phi(a) + \ell_m(a)/M$
    \ENDFOR
\ENDFOR

\STATE \textbf{Act:}
\STATE $\displaystyle
       \rho(a \mid s,u)
       \leftarrow
       \frac{\exp(\beta \phi(a))}
       {\sum_{a' \in \mathcal{A}(s)}
        \exp(\beta \phi(a'))}
       \quad
       \forall a \in \mathcal{A}(s)$
\STATE Sample $a^\star \sim \rho(\cdot \mid s,u)$
\RETURN $a^\star$
\end{algorithmic}
\end{algorithm}

\section{Proof and theory}
\label{app:theory}

Our method uses the model's reverse conditional $\pist(u\mid s,a)$, which masked dLLMs evaluate natively. 
We assume the two conditionals are consistent views of one underlying joint.

\begin{assumption}[Coherent conditionals]
\label{ass:coherence}
The model's forward and reverse conditionals satisfy Bayes' rule with respect to the task prior $p(u\mid s)$: for every $a$ and every $u$ in the support of $p(\cdot\mid s)$,
\begin{equation}
\label{eq:coherence}
\pist(u\mid s,a)
\;=\;
\frac{p(u\mid s)\,\pist(a\mid s,u)}
{\sum_{u'}p(u'\mid s)\,\pist(a\mid s,u')}.
\end{equation}
\end{assumption}

Assumption~\ref{ass:coherence} is an idealization of the same kind used whenever a single model's conditionals are composed; it is what licenses reading the reverse score as a posterior. 
A stronger variant, in which $\pist(u\mid s,a)$ is coherent with the task posterior $p(u\mid s,a)$ as well, removes the task prior from the proposition entirely (Remark~\ref{rem:full-coherence}).

\subsection{Proof of the recovery proposition}
\label{app:theory_prop}
This appendix proves Proposition~\ref{thm:main}. The argument has two steps. 
The first step (Lemma~\ref{lem:factor}) isolates the only property of the task prior that matters: it can be changed arbitrarily without moving the induced task posterior, provided the corruption is re-absorbed accordingly. 
The second step substitutes Assumption~\ref{ass:corruption} and lets the corruption cancel.

\begin{lemma}[Prior rescaling leaves the posterior untouched]
\label{lem:factor}
Let $q(u\mid s)$ and $q'(u\mid s)$ be two task priors with the same support,
and let $\nu(a\mid s,u)$ be any conditional. Define
\begin{equation}
\label{eq:induced}
\nu_q(u\mid s,a)
\;:=\;
\frac{q(u\mid s)\,\nu(a\mid s,u)}{\sum_{u'}q(u'\mid s)\,\nu(a\mid s,u')},
\end{equation}
and $\nu_{q'}(u\mid s,a)$ analogously with $q'$ in place of $q$. 
Then $\nu_q(u\mid s,a)=\nu_{q'}(u\mid s,a)$ for every $a$ if $q'(u\mid s)/q(u\mid s)$ is constant in $u$ on the support; in particular the induced posterior is unchanged whenever $q'$ is replaced by $q$ itself.
\end{lemma}

\begin{proof}
Immediate from Eq.\ref{eq:induced}: replacing $q$ by $q'$ multiplies the numerator at $u$ by $q'(u\mid s)/q(u\mid s)$ and the denominator by the same factor averaged over $u'$; the two coincide for every $a$ exactly when the ratio does not depend on $u$.
The trivial case $q'=q$ gives the invariance used below.
\end{proof}

\begin{proposition*}[Proposition \ref{thm:main} restated; exact recovery of the faithful task posterior]
Let Assumptions~\ref{ass:corruption} and \ref{ass:coherence} hold. 
Then, for every state $s$, task $u$ in the support of $p(\cdot\mid s)$, and action $a$,
\begin{equation}
\label{eq:master-app}
\pist(u\mid s,a)
\;=\;
\frac{\omega(u\mid s)\,\mu(a\mid s,u)}
{\displaystyle\sum_{u'}\omega(u'\mid s)\,\mu(a\mid s,u')},
\end{equation}
where
\begin{equation}
\label{eq:weights-app}
\omega(u\mid s)\;:=\;\frac{p(u\mid s)}{Z(s,u)}
\end{equation}
is a positive task weight that does not depend on the candidate action $a$, and $Z(s,u)=\sum_{a'}\mu(a'\mid s,u)\,e^{\varepsilon(s,a')}$ is the normalizer of Assumption~\ref{ass:corruption}.
\end{proposition*}

\begin{proof}
By Assumption~\ref{ass:coherence}, the model's reverse conditional satisfies
\begin{equation}
\label{eq:coherence-app}
\pist(u\mid s,a)
\;=\;
\frac{p(u\mid s)\,\pist(a\mid s,u)}
{\sum_{u'}p(u'\mid s)\,\pist(a\mid s,u')}.
\end{equation}
Apply Assumption~\ref{ass:corruption} to every term. The numerator at $u$ becomes
\begin{equation}
\label{eq:num-app}
p(u\mid s)\,\pist(a\mid s,u)
\;=\;
p(u\mid s)\,\frac{\mu(a\mid s,u)\,e^{\varepsilon(s,a)}}{Z(s,u)},
\end{equation}
and the denominator becomes
\begin{equation}
\label{eq:den-app}
\sum_{u'}p(u'\mid s)\,\pist(a\mid s,u')
\;=\;
\sum_{u'}p(u'\mid s)\,
\frac{\mu(a\mid s,u')\,e^{\varepsilon(s,a)}}{Z(s,u')}
\;=\;
e^{\varepsilon(s,a)}
\sum_{u'}\frac{p(u'\mid s)}{Z(s,u')}\,\mu(a\mid s,u'),
\end{equation}
where the factor $e^{\varepsilon(s,a)}$ carries no task index $u'$ and factors out of the sum. 
Substituting Eq.\ref{eq:num-app} and Eq.\ref{eq:den-app} into Eq.\ref{eq:coherence-app}, the factor $e^{\varepsilon(s,a)}$ appears once in the numerator and once in the denominator and cancels exactly:
\begin{equation*}
\pist(u\mid s,a)
\;=\;
\frac{\bigl(p(u\mid s)/Z(s,u)\bigr)\,\mu(a\mid s,u)}
{\sum_{u'}\bigl(p(u'\mid s)/Z(s,u')\bigr)\,\mu(a\mid s,u')}
\;=\;
\frac{\omega(u\mid s)\,\mu(a\mid s,u)}
{\sum_{u'}\omega(u'\mid s)\,\mu(a\mid s,u')},
\end{equation*}
which is Eq.\ref{eq:master-app}. 
Lemma~\ref{lem:factor} is what licenses reading the right-hand side as a task posterior: the prior has been rescaled from $p$ to $\omega$, but the induced posterior remains a genuine Bayes posterior of $\mu$, simply under the prior $\omega$.
\end{proof}

Identity Eq.\ref{eq:master-app} holds for every action individually, so it preserves every feature of the comparison of actions: for any $a_1,a_2$,
\begin{align}
\label{eq:margins-app}
& \log\pist(u\mid s,a_1)-\log\pist(u\mid s,a_2)
\;=\; \\
& \log\frac{\mu(a_1\mid s,u)}{\sum_{u'}\omega(u'\mid s)\,\mu(a_1\mid s,u')}
\;-\;
\log\frac{\mu(a_2\mid s,u)}{\sum_{u'}\omega(u'\mid s)\,\mu(a_2\mid s,u')},
\end{align}
with equal margins, equal ties, and hence the same ordering, top-$k$ set, and argmax as the faithful model's task-evidence ratio; the margins are those of the ratio, and need not coincide with the margins of $\mu(\cdot\mid s,u)$ itself (Remark~\ref{rem:ratio-vs-conditional}).

\begin{corollary}[The corrected policy is the softmax of the faithful task-evidence ratio, at any temperature]
\label{cor:policy-app}
Let $\mathcal{A}(s)$ be a finite candidate set and $\beta>0$ an inverse temperature. Under Assumptions~\ref{ass:corruption} and \ref{ass:coherence},
\begin{equation}
\label{eq:policy-app}
\frac{\pist(u\mid s,a)^{\beta}}
{\sum_{a'\in\mathcal{A}(s)}\pist(u\mid s,a')^{\beta}}
\;=\;
\frac{\bigl(\mu(a\mid s,u)\big/\sum_{u'}\omega(u'\mid s)\,\mu(a\mid s,u')\bigr)^{\beta}}
{\sum_{a'\in\mathcal{A}(s)}
\bigl(\mu(a'\mid s,u)\big/\sum_{u'}\omega(u'\mid s)\,\mu(a'\mid s,u')\bigr)^{\beta}}
\end{equation}
for every $a\in\mathcal{A}(s)$: an exact equality of probability distributions over candidates.
The limit $\beta\to\infty$ recovers the argmax statement; $\beta\to 0$ recovers uniform sampling.
\end{corollary}

\begin{proof}
By Proposition~\ref{thm:main}, for each $a$ the score $\pist(u\mid s,a)$ equals the faithful task-evidence ratio times the factor $\omega(u\mid s)$, which does not depend on $a$. 
Raising to $\beta$ and normalizing over $\mathcal{A}(s)$ removes this shared factor from both sides.
\end{proof}

\begin{remark}[Ratio versus conditional]
\label{rem:ratio-vs-conditional}
The right-hand side of Eq.\ref{eq:policy-app} is the softmax of the faithful task-evidence ratio $\mu(a\mid s,u)\big/\sum_{u'}\omega(u'\mid s)\,\mu(a\mid s,u')$, not of the faithful conditional $\mu(a\mid s,u)$ itself: the task-marginal likelihood $m(a)=\sum_{u'}\omega(u'\mid s)\,\mu(a\mid s,u')$ depends on the candidate and is not removed by normalization.
The two policies coincide when $m$ is constant over $\mathcal{A}(s)$, e.g. when the candidates are equally plausible under the task mixture, or in the argmax limit $\beta\to\infty$, where only the ordering survives.
Unconditionally, what is recovered exactly is the faithful task-evidence ratio's comparison of candidates: ordering and ties.
\end{remark}

\begin{remark}[On the weights $\omega$]
\label{rem:weights-app}
The faithful baseline on the right of Eq.\ref{eq:master-app} mixes tasks with weights $\omega(u\mid s)=p(u\mid s)/Z(s,u)$ rather than with the prior $p(u\mid s)$ itself, because the corruption rescales each task-conditional by its own normalizer $Z(s,u)$. 
Only two properties of $\omega$ are used: it is positive and it does not depend on the candidate action $a$. 
If $Z(s,u)$ happens to be constant in $u$, i.e. the corruption rescales every task-conditional by the same total mass, then $\omega\propto p$ and the baseline is exactly the task-marginal of $\mu$ under the original prior.
\end{remark}

\begin{remark}[Dropping the task prior]
\label{rem:full-coherence}
If the model's conditionals are coherent with the task posterior as well, i.e. $\pist(u\mid s,a)$ directly proportional to $p(u\mid s,a)\,\pist(a\mid s,u)$, then the same cancellation applies to $p(u\mid s,a)$ itself, and Proposition~\ref{thm:main} holds with $\omega(u\mid s)$ replaced by $p(u\mid s,a)/Z(s,u)$.
In this case the method requires only the single reverse score $\pist(u\mid s,a)$: no explicit task prior and no baseline forward pass enter at all.
\end{remark}

\subsection{Residual bias of the corrected policy}
\label{app:bias}

Remark~\ref{rem:ratio-vs-conditional} shows that the corrected policy $\rho$ is exactly the softmax of the faithful task-evidence ratio, and asks how far it remains from the faithful conditional itself.
This appendix answers the question: it proves Theorem~\ref{thm:bias} of Section~\ref{sec:method-bias}, via a sharp range lemma (Lemma~\ref{lem:range}).
Throughout, fix a state $s$, a task $u$ in the support of $p(\cdot\mid s)$, and a finite candidate set $\cA\subseteq\supp\mu(\cdot\mid s,u)$ with $|\cA|\ge2$; recall $\rho$ from Eq.\ref{eq:policy}, the positive task weights $\omega$ from Proposition~\ref{thm:main}, the task-marginal likelihood $m(a)=\sum_{u'}\omega(u'\mid s)\,\mu(a\mid s,u')$, and the restricted distributions and spreads of Eq.\ref{eq:lam-task}.

\paragraph{Support identity.}
Since $e^{\varepsilon(s,a)}>0$, Assumption~\ref{ass:corruption} gives $\pist(a\mid s,u)>0$ if and only if $\mu(a\mid s,u)>0$.
Hence a candidate set proposed from $\pist(\cdot\mid s,u)$ automatically satisfies $\cA\subseteq\supp\mu(\cdot\mid s,u)$, and on $\cA$ we have $\mu(a\mid s,u)>0$ as well as $m(a)\ge\omega(u\mid s)\,\mu(a\mid s,u)>0$.
Moreover, by Eq.\ref{eq:corruption},
\begin{equation}
\label{eq:piA}
\pi_{\cA}(a)
\;=\;
\frac{\mu_{\cA}(a)\,e^{\varepsilon(s,a)}}{\sum_{a'\in\cA}\mu_{\cA}(a')\,e^{\varepsilon(s,a')}},
\qquad a\in\cA,
\end{equation}
since the normalizer $Z(s,u)$ of Eq.\ref{eq:corruption} cancels between numerator and denominator: the corrupted restriction is the faithful restriction tilted by the corruption.

\paragraph{A sharp range bound.}
Lemma~\ref{lem:range} below quantifies a simple fact: two distributions whose likelihood ratio is confined to a band cannot be far apart in total variation, and the worst case for a band of logarithmic width $\Delta$ is exactly $\tanh(\Delta/4)$.
The key input is a sharp result of \citet{Binette_2019} (Corollary~5): if the likelihood ratio $dP/dQ$ has essential infimum $m$ and essential supremum $M$, then $\dtv(P,Q)\le\frac{(M-1)(1-m)}{M-m}$, with the best possible constant for fixed $(m,M)$.
Lemma~\ref{lem:range} is what this bound becomes under the scale-invariant constraint that only the \emph{range} of $\log(dQ/dP)$, and not the location of the band, is known.

\begin{lemma}[Range bound, sharp]
\label{lem:range}
Let $P$ and $Q$ be strictly positive probability distributions on a finite set $S$, and suppose the log-likelihood ratio has range at most $\Delta$:
\begin{equation}
\label{eq:range}
\max_{x\in S}\log\frac{Q(x)}{P(x)}
\;-\;
\min_{x\in S}\log\frac{Q(x)}{P(x)}
\;\le\;\Delta .
\end{equation}
Then
\begin{equation}
\label{eq:tvbound}
\dtv(P,Q)\;\le\;\tanh\!\bigl(\Delta/4\bigr),
\qquad\text{where}\qquad
\dtv(P,Q):=\tfrac12\sum_{x\in S}|P(x)-Q(x)|.
\end{equation}
The constant is sharp: for every $\Delta$ there is a two-point family attaining equality.
\end{lemma}

Lemma~\ref{lem:range} is obtained by Corollary 5 of \citet{Binette_2019}, followed by optimizing over the location of the bounded likelihood-ratio interval.

\begin{proof}[Proof of Theorem~\ref{thm:bias}]
\emph{(i)} By Eq.\ref{eq:policy-app}, $\rho(a)=R(a)\big/\sum_{a'\in\cA}R(a')$ with $R(a)=(\mu(a\mid s,u)/m(a))^{\beta}$.
Hence for every $a\in\cA$,
\begin{equation}
\label{eq:logratio}
\log\frac{\rho(a)}{\mu_{\cA}(a)}
=(\beta-1)\log\mu(a\mid s,u)-\beta\log m(a)
+\underbrace{\Bigl(\log\mu(\cA\mid s,u)-\log\textstyle\sum_{a'\in\cA}R(a')\Bigr)}_{\text{constant in }a}.
\end{equation}
The range of $\log(\rho/\mu_{\cA})$ over $\cA$ is therefore the range of $(\beta-1)\log\mu(\cdot\mid s,u)-\beta\log m$, which is at most $|\beta-1|\,\Lambda_\mu+\beta\Lambda_m$.
Lemma~\ref{lem:range}, applied to $P=\mu_{\cA}$ and $Q=\rho$ (both strictly positive on $\cA$ by the support identity), gives the first inequality in Eq.\ref{eq:revbound}.

For the second inequality we show $\Lambda_m\le\Lambda_{\mathrm{task}}$.
Fix $a_1,a_2\in\cA$ and let $S_2:=\{u'\in\supp p(\cdot\mid s):\mu(a_2\mid s,u')>0\}$.
If some $u'\notin S_2$ has $\mu(a_1\mid s,u')>0$, then $\log\mu(a_1\mid s,u')-\log\mu(a_2\mid s,u')=+\infty$, so $\Lambda_{\mathrm{task}}=+\infty$ and the claim is vacuous.
Otherwise $\mu(a_1\mid s,u')=0$ off $S_2$, and the mediant inequality (a weighted mean of ratios lies below the largest ratio) gives
\begin{equation*}
\frac{m(a_1)}{m(a_2)}
=\frac{\sum_{u'\in S_2}\omega(u'\mid s)\,\mu(a_2\mid s,u')\,\bigl[\mu(a_1\mid s,u')/\mu(a_2\mid s,u')\bigr]}
{\sum_{u'\in S_2}\omega(u'\mid s)\,\mu(a_2\mid s,u')}
\;\le\;
\max_{u'\in S_2}\frac{\mu(a_1\mid s,u')}{\mu(a_2\mid s,u')},
\end{equation*}
so $\log m(a_1)-\log m(a_2)\le\max_{u'}\bigl[\log\mu(a_1\mid s,u')-\log\mu(a_2\mid s,u')\bigr]\le\Lambda_{\mathrm{task}}$.
Taking the maximum over $a_1,a_2\in\cA$ yields $\Lambda_m\le\Lambda_{\mathrm{task}}$, and $\tanh$ is increasing.

\medskip
\emph{(ii)} By Eq.\ref{eq:piA}, writing $\varepsilon(a):=\varepsilon(s,a)$ and $\varepsilon^{+}:=\varepsilon(a^{+})$,
\begin{equation*}
\pi_{\cA}(a^{+})
=\frac{\gamma\,e^{\varepsilon^{+}}}{\sum_{a\in\cA}\mu_{\cA}(a)e^{\varepsilon(a)}}
\;\ge\;
\frac{\gamma\,e^{\varepsilon^{+}}}{\gamma e^{\varepsilon^{+}}+(1-\gamma)e^{\varepsilon^{+}-\Delta_2}}
=\frac{\gamma}{\gamma+(1-\gamma)e^{-\Delta_2}},
\end{equation*}
since $\varepsilon(a)\le\varepsilon^{+}-\Delta_2$ for all $a\ne a^{+}$ by definition of $\Delta_2$; equality holds when $|\cA|=2$, as there is only one other candidate.
Moreover $\pi_{\cA}(a^{+})\ge\gamma=\mu_{\cA}(a^{+})$, because $\sum_a\mu_{\cA}(a)e^{\varepsilon(a)}\le e^{\varepsilon^{+}}$.
Taking the event $E=\{a^{+}\}$ in $\dtv(\mu_{\cA},\pi_{\cA})=\sup_E\bigl(\pi_{\cA}(E)-\mu_{\cA}(E)\bigr)$,
\begin{equation*}
\dtv(\mu_{\cA},\pi_{\cA})
\;\ge\;\pi_{\cA}(a^{+})-\gamma
\;\ge\;\frac{\gamma}{\gamma+(1-\gamma)e^{-\Delta_2}}-\gamma
=\frac{\gamma(1-\gamma)\bigl(1-e^{-\Delta_2}\bigr)}{\gamma+(1-\gamma)e^{-\Delta_2}} .
\end{equation*}
For the monotonicity claim, write $x=e^{-\Delta_2}$: the right-hand side is $\gamma(1-\gamma)(1-x)/(\gamma+(1-\gamma)x)$, with derivative $-\gamma(1-\gamma)/(\gamma+(1-\gamma)x)^{2}<0$ in $x$; hence it increases in $\Delta_2$ and tends to $\gamma(1-\gamma)/\gamma=1-\gamma$ as $x\to0$.

\end{proof}

\begin{remark}[Finite-$M$ estimation enters additively]
\label{rem:finite-m}
Theorem~\ref{thm:bias} analyzes the exact reverse score $\phi(a)=\log\pist(u\mid s,a)$, as does Proposition~\ref{thm:main}; the algorithm uses the Monte Carlo estimate $\hat\phi$ of Eq.\ref{eq:score}.
If the estimation error $\delta(a):=\hat\phi(a)-\phi(a)$ satisfies $\sup_{a\in\cA}|\delta(a)|\le\delta_{\max}$, then for the estimated policy $\hat\rho(a)\propto e^{\beta\hat\phi(a)}$,
\begin{equation*}
\log\frac{\hat\rho(a)}{\rho(a)}
=\beta\,\delta(a)-\log\E_{\rho}\bigl[e^{\beta\delta}\bigr]
\qquad\Longrightarrow\qquad
\dtv(\hat\rho,\rho)\le\tanh\!\bigl(\beta\delta_{\max}/2\bigr),
\end{equation*}
by Lemma~\ref{lem:range}, since the range of $\beta\delta$ over $\cA$ is at most $2\beta\delta_{\max}$.
Hence $\dtv(\hat\rho,\mu_{\cA})\le\tanh(\beta\delta_{\max}/2)+B$: estimation noise degrades the corruption-free guarantee gracefully and vanishes as the estimator becomes consistent.
\end{remark}

\clearpage
\section{Hyperparameters}
\label{app:hyperparameters}
Table~\ref{tab:hyperparameters} summarizes the hyperparameters used for
action generation and reflection. On ScienceWorld and Jericho,
\textbf{Reflect Reverse} uses more Monte Carlo sampling steps than
\textbf{Reflect Forward} because the task description is longer than
the candidate actions at each interaction step, resulting in more tokens
to evaluate during reverse reflection.

Due to iLLaDA's pre-training configuration, the assistant prefill and end-of-thinking sequence is
\texttt{</think>} followed by a newline and \texttt{Thought:}, while the
visible response begins with \texttt{Thought:} followed by a space.
\begin{table*}[!htp]
    \centering
    \small
    \setlength{\tabcolsep}{5pt}
    \renewcommand{\arraystretch}{1.12}
    \caption{Evaluation hyperparameters for LLaDA and iLLaDA on
    ALFWorld, ScienceWorld, BabyAI, and Jericho. The upper table reports
    model settings and the lower table reports forward
    and reverse likelihood-scoring parameters}
    \label{tab:hyperparameters}
    \begin{tabular}{@{}lcc@{}}
        \toprule
        \textbf{Hyperparameters} & \textbf{LLaDA} & \textbf{iLLaDA} \\
        \midrule
        Generation length (tokens)               & 128   & 128 \\
        Requested diffusion steps                & 128   & 128 \\
        Block size (tokens)                      & 32    & 32 \\
        Threshold                                & 0.9   & 0.9 \\
        Dual cache                               & On    & Off \\
        Configured context length (tokens)       & 4,000 & 8,192 \\
        Temperature                             & $\{0.8, 0.9\}$ & $\{0.9, 1.0\}$ \\
        Seeds                                   & $\{12, 22, 32\}$ & $\{12, 22, 32\}$ \\
        Maximum agent steps per episode          & 30    & 30 \\
        Generation attempt limit                 & 20    & 20 \\
        \bottomrule
    \end{tabular}

    \par\medskip

    \begin{tabular}{@{}lcccc@{}}
        \toprule
        & \multicolumn{2}{c}{\textbf{Reflect Forward}}
        & \multicolumn{2}{c}{\textbf{Reflect Reverse}} \\
        \cmidrule(lr){2-3}\cmidrule(l){4-5}
        \textbf{Benchmark} & \textbf{MC rollouts} & \textbf{Batch size}
        & \textbf{MC rollouts} & \textbf{Batch size} \\
        \midrule
        ALFWorld     & 64 & 16 & 64  & 16 \\
        ScienceWorld & 64 & 16 & 128 & 16 \\
        BabyAI       & 64 & 16 & 64  & 16 \\
        Jericho      & 64 & 16 & 128 & 16 \\
        \bottomrule
    \end{tabular}
\end{table*}

\clearpage
\section{Supplementary experiment results}
\label{app:rsupplementary}

We provide the full experimental results omitted from the main paper for space. Tables~\ref{tab:agentboard-results} and~\ref{tab:agentboard-iLLaDA-results} report performance across all benchmarks, model variants, methods, and sampling temperatures, offering a more detailed view of the results summarized in the main text.

\begin{table}[!htbp]
  \centering
  \caption{
    Success rate (SR) and progress rate (PR) evaluated by LLaDA-8B-Instruct.
    Entries are mean $\pm$ sample standard deviation.
    Boldface marks the best method within each benchmark and
    temperature category (including Mean), with ties included.
  }
  \label{tab:agentboard-results}

  \begingroup
  \footnotesize
  \setlength{\tabcolsep}{2.5pt}
  \renewcommand{\arraystretch}{1.02}

  \begin{tabularx}{\linewidth}{@{}lXcrr@{}}
    \toprule[1pt]
    Benchmark & Method & Temperature & SR (\%) & PR (\%) \\
    \midrule

    \textbf{ALFWorld}
      & One Pass & 0.8
      & $6.47 \pm 2.40$
      & $26.87 \pm 1.46$ \\

      & & 0.9
      & $8.71 \pm 1.55$
      & $29.33 \pm 0.79$ \\

      & & \emph{Mean}
      & $7.59 \pm 2.18$
      & $28.10 \pm 1.71$ \\

    \cmidrule[0.3pt]{2-5}

      & Reflect Forward & 0.8
      & $7.46 \pm 1.49$
      & $29.73 \pm 2.18$ \\

      & & 0.9
      & $8.96 \pm 0.75$
      & $31.45 \pm 1.02$ \\

      & & \emph{Mean}
      & $8.21 \pm 1.33$
      & $30.59 \pm 1.79$ \\

    \cmidrule[0.3pt]{2-5}

      & \textbf{Reflect Reverse} & 0.8
      & $\mathbf{12.19} \pm 2.28$
      & $\mathbf{34.29} \pm 3.03$ \\

      & & 0.9
      & $\mathbf{16.17} \pm 2.28$
      & $\mathbf{36.24} \pm 1.01$ \\

      & & \emph{Mean}
      & $\mathbf{14.18} \pm 2.99$
      & $\mathbf{35.26} \pm 2.29$ \\

    \specialrule{1pt}{2.5pt}{2.5pt}

    \textbf{ScienceWorld}
      & One Pass & 0.8
      & $3.70 \pm 1.70$
      & $21.69 \pm 4.42$ \\

      & & 0.9
      & $4.07 \pm 0.64$
      & $\mathbf{23.02} \pm 1.21$ \\

      & & \emph{Mean}
      & $3.89 \pm 1.17$
      & $22.36 \pm 2.99$ \\

    \cmidrule[0.3pt]{2-5}

      & Reflect Forward & 0.8
      & $3.70 \pm 1.28$
      & $20.03 \pm 2.37$ \\

      & & 0.9
      & $5.56 \pm 1.92$
      & $22.26 \pm 2.34$ \\

      & & \emph{Mean}
      & $4.63 \pm 1.78$
      & $21.15 \pm 2.43$ \\

    \cmidrule[0.3pt]{2-5}

      & \textbf{Reflect Reverse} & 0.8
      & $\mathbf{7.78} \pm 1.92$
      & $\mathbf{23.24} \pm 2.91$ \\

      & & 0.9
      & $\mathbf{5.93} \pm 1.28$
      & $22.41 \pm 1.26$ \\

      & & \emph{Mean}
      & $\mathbf{6.85} \pm 1.78$
      & $\mathbf{22.82} \pm 2.06$ \\

    \specialrule{1pt}{2.5pt}{2.5pt}

    \textbf{BabyAI}
      & One Pass & 0.8
      & $14.88 \pm 0.52$
      & $25.97 \pm 0.66$ \\

      & & 0.9
      & $14.29 \pm 0.89$
      & $27.46 \pm 1.02$ \\

      & & \emph{Mean}
      & $14.58 \pm 0.73$
      & $26.71 \pm 1.12$ \\

    \cmidrule[0.3pt]{2-5}

      & Reflect Forward & 0.8
      & $\mathbf{18.75} \pm 4.72$
      & $\mathbf{30.41} \pm 3.19$ \\

      & & 0.9
      & $16.37 \pm 1.03$
      & $\mathbf{28.39} \pm 0.77$ \\

      & & \emph{Mean}
      & $\mathbf{17.56} \pm 3.32$
      & $\mathbf{29.40} \pm 2.35$ \\

    \cmidrule[0.3pt]{2-5}

      & \textbf{Reflect Reverse} & 0.8
      & $17.86 \pm 1.79$
      & $26.55 \pm 1.17$ \\

      & & 0.9
      & $\mathbf{17.26} \pm 3.61$
      & $27.12 \pm 1.72$ \\

      & & \emph{Mean}
      & $\mathbf{17.56} \pm 2.57$
      & $26.83 \pm 1.35$ \\

    \specialrule{1pt}{2.5pt}{2.5pt}

    \textbf{Jericho}
      & One Pass & 0.8
      & $\mathbf{1.67} \pm 2.89$
      & $19.59 \pm 6.37$ \\

      & & 0.9
      & $0.00 \pm 0.00$
      & $20.35 \pm 3.17$ \\

      & & \emph{Mean}
      & $0.83 \pm 2.04$
      & $19.97 \pm 4.52$ \\

    \cmidrule[0.3pt]{2-5}

      & Reflect Forward & 0.8
      & $0.00 \pm 0.00$
      & $17.52 \pm 2.50$ \\

      & & 0.9
      & $\mathbf{3.33} \pm 2.89$
      & $20.94 \pm 1.39$ \\

      & & \emph{Mean}
      & $\mathbf{1.67} \pm 2.58$
      & $19.23 \pm 2.61$ \\

    \cmidrule[0.3pt]{2-5}

      & \textbf{Reflect Reverse} & 0.8
      & $0.00 \pm 0.00$
      & $\mathbf{20.53} \pm 2.22$ \\

      & & 0.9
      & $\mathbf{3.33} \pm 5.77$
      & $\mathbf{26.72} \pm 2.47$ \\

      & & \emph{Mean}
      & $\mathbf{1.67} \pm 4.08$
      & $\mathbf{23.63} \pm 3.99$ \\

    \bottomrule[1pt]
  \end{tabularx}

  \endgroup
\end{table}

\begin{table}[!htbp]
  \centering
  \caption{
    Success rate (SR) and progress rate (PR) evaluated by iLLaDA-8B-Instruct.
    Entries are mean $\pm$ sample standard deviation.
    Boldface marks the best method within each benchmark and
    temperature category (including Mean), with ties included.
  }
  \label{tab:agentboard-iLLaDA-results}

  \begingroup
  \footnotesize
  \setlength{\tabcolsep}{2.5pt}
  \renewcommand{\arraystretch}{1.02}

  \begin{tabularx}{\linewidth}{@{}lXcrr@{}}
    \toprule[1pt]
    Benchmark & Method & Temperature & SR (\%) & PR (\%) \\
    \midrule

    \textbf{ALFWorld}
      & One Pass & 0.9
      & $5.47 \pm 0.43$
      & $19.57 \pm 1.76$ \\

      & & 1.0
      & $4.73 \pm 0.43$
      & $21.06 \pm 1.81$ \\

      & & \emph{Mean}
      & $5.10 \pm 0.56$
      & $20.32 \pm 1.79$ \\

    \cmidrule[0.3pt]{2-5}

      & Reflect Forward & 0.9
      & $0.75 \pm 0.75$
      & $15.80 \pm 1.39$ \\

      & & 1.0
      & $0.75 \pm 1.29$
      & $16.36 \pm 1.97$ \\

      & & \emph{Mean}
      & $0.75 \pm 0.94$
      & $16.08 \pm 1.56$ \\

    \cmidrule[0.3pt]{2-5}

      & \textbf{Reflect Reverse} & 0.9
      & $\mathbf{9.95} \pm 1.88$
      & $\mathbf{27.69} \pm 2.86$ \\

      & & 1.0
      & $\mathbf{10.20} \pm 1.14$
      & $\mathbf{28.71} \pm 1.55$ \\

      & & \emph{Mean}
      & $\mathbf{10.07} \pm 1.40$
      & $\mathbf{28.20} \pm 2.13$ \\

    \specialrule{1pt}{2.5pt}{2.5pt}

    \textbf{ScienceWorld}
      & One Pass & 0.9
      & $0.00 \pm 0.00$
      & $15.77 \pm 0.22$ \\

      & & 1.0
      & $0.37 \pm 0.64$
      & $16.47 \pm 1.73$ \\

      & & \emph{Mean}
      & $0.19 \pm 0.45$
      & $16.12 \pm 1.17$ \\

    \cmidrule[0.3pt]{2-5}

      & Reflect Forward & 0.9
      & $0.37 \pm 0.64$
      & $7.70 \pm 1.21$ \\

      & & 1.0
      & $0.00 \pm 0.00$
      & $9.82 \pm 1.48$ \\

      & & \emph{Mean}
      & $0.19 \pm 0.45$
      & $8.76 \pm 1.68$ \\

    \cmidrule[0.3pt]{2-5}

      & \textbf{Reflect Reverse} & 0.9
      & $\mathbf{3.70} \pm 1.28$
      & $\mathbf{25.95} \pm 2.43$ \\

      & & 1.0
      & $\mathbf{4.07} \pm 2.31$
      & $\mathbf{25.39} \pm 0.84$ \\

      & & \emph{Mean}
      & $\mathbf{3.89} \pm 1.69$
      & $\mathbf{25.67} \pm 1.66$ \\

    \specialrule{1pt}{2.5pt}{2.5pt}

    \textbf{BabyAI}
      & One Pass & 0.9
      & $9.23 \pm 2.25$
      & $22.65 \pm 2.07$ \\

      & & 1.0
      & $11.61 \pm 2.36$
      & $25.13 \pm 2.82$ \\

      & & \emph{Mean}
      & $10.42 \pm 2.44$
      & $23.89 \pm 2.60$ \\

    \cmidrule[0.3pt]{2-5}

      & Reflect Forward & 0.9
      & $8.93 \pm 0.89$
      & $21.52 \pm 1.15$ \\

      & & 1.0
      & $8.04 \pm 1.55$
      & $18.44 \pm 2.49$ \\

      & & \emph{Mean}
      & $8.48 \pm 1.23$
      & $19.98 \pm 2.42$ \\

    \cmidrule[0.3pt]{2-5}

      & \textbf{Reflect Reverse} & 0.9
      & $\mathbf{15.77} \pm 2.25$
      & $\mathbf{30.73} \pm 1.47$ \\

      & & 1.0
      & $\mathbf{20.24} \pm 1.36$
      & $\mathbf{35.40} \pm 0.36$ \\

      & & \emph{Mean}
      & $\mathbf{18.01} \pm 2.96$
      & $\mathbf{33.07} \pm 2.73$ \\

    \specialrule{1pt}{2.5pt}{2.5pt}

    \textbf{Jericho}
      & One Pass & 0.9
      & $\mathbf{1.67} \pm 2.89$
      & $11.58 \pm 2.39$ \\

      & & 1.0
      & $\mathbf{1.67} \pm 2.89$
      & $11.62 \pm 2.61$ \\

      & & \emph{Mean}
      & $\mathbf{1.67} \pm 2.58$
      & $11.60 \pm 2.24$ \\

    \cmidrule[0.3pt]{2-5}

      & Reflect Forward & 0.9
      & $0.00 \pm 0.00$
      & $9.63 \pm 1.31$ \\

      & & 1.0
      & $0.00 \pm 0.00$
      & $11.51 \pm 1.08$ \\

      & & \emph{Mean}
      & $0.00 \pm 0.00$
      & $10.57 \pm 1.49$ \\

    \cmidrule[0.3pt]{2-5}

      & \textbf{Reflect Reverse} & 0.9
      & $0.00 \pm 0.00$
      & $\mathbf{11.93} \pm 1.25$ \\

      & & 1.0
      & $\mathbf{1.67} \pm 2.89$
      & $\mathbf{21.37} \pm 2.69$ \\

      & & \emph{Mean}
      & $0.83 \pm 2.04$
      & $\mathbf{16.65} \pm 5.50$ \\

    \bottomrule[1pt]
  \end{tabularx}

  \endgroup
\end{table}

\FloatBarrier

\section{Retry state examples}
\label{app:retrystateexample}
Tables~\ref{tab:alfworld-retry-64} and~\ref{tab:alfworld-retry-23} present two representative retry-state patterns from ALFWorld: single-action repetition and two-action alternation. Similar retry patterns are also frequently observed across other multi-turn benchmarks.
\begingroup
\footnotesize

\definecolor{retryred}{HTML}{C62828}

\setlength{\tabcolsep}{6pt}
\setlength{\LTleft}{0pt plus \textwidth}
\setlength{\LTright}{0pt plus \textwidth}
\setlength{\LTcapwidth}{\linewidth}

\begin{longtable}{
  @{}>{\raggedright\arraybackslash}p{\linewidth}@{}}

\caption{Single-action example. A and O denote the action and its resulting observation. The repetition loop is marked in red including the initial action of the loop.}
\label{tab:alfworld-retry-64}\\

\toprule
\textbf{Look at pencil under the desklamp.} \\
\midrule
\endfirsthead

\toprule
\textbf{look at pencil under the desklamp.}
\textit{(continued)} \\
\midrule
\endhead

\midrule
\multicolumn{1}{r}{
  \textit{Continued on next page}
}\\
\endfoot

\bottomrule
\endlastfoot

{
  \textbf{A:} \texttt{go to desk 1}
  \par
  \textbf{O:}
  On the desk 1, you see a bowl 2, a creditcard 2,
  a desklamp 1, a laptop 1, a mug 1, and a pen 2.
  \par
}\\[3pt]
\mbox{}\\

{
  \textbf{A:} \texttt{examine desk 1}
  \par
  \textbf{O:}
  On the desk 1, you see a bowl 2, a creditcard 2,
  a desklamp 1, a laptop 1, a mug 1, and a pen 2.
  \par
}\\[3pt]
\mbox{}\\

{
  \textbf{A:} \texttt{go to desk 2}
  \par
  \textbf{O:}
  On the desk 2, you see a alarmclock 1,
  a bowl 3, and a pencil 1.
  \par
}\\[3pt]
\mbox{}\\

{
  \color{retryred}
  \textbf{A:} \texttt{examine desk 2}
  \par
  \textbf{O:}
  On the desk 2, you see a alarmclock 1,
  a bowl 3, and a pencil 1.
  \par
}\\[3pt]
\mbox{}\\

{
  \color{retryred}
  \textbf{A:} \texttt{examine desk 2}
  \par
  \textbf{O:}
  On the desk 2, you see a alarmclock 1,
  a bowl 3, and a pencil 1.
  \par
}\\[3pt]
\mbox{}\\

{
  \color{retryred}
  \textbf{A:} \texttt{examine desk 2}
  \par
  \textbf{O:}
  On the desk 2, you see a alarmclock 1,
  a bowl 3, and a pencil 1.
  \par
}\\[3pt]
\mbox{}\\

{
  \color{retryred}
  \textbf{A:} \texttt{examine desk 2}
  \par
  \textbf{O:}
  On the desk 2, you see a alarmclock 1,
  a bowl 3, and a pencil 1.
  \par
}\\[3pt]
\mbox{}\\

{
  \color{retryred}
  \textbf{A:} \texttt{examine desk 2}
  \par
  \textbf{O:}
  On the desk 2, you see a alarmclock 1,
  a bowl 3, and a pencil 1.
  \par
}\\[3pt]
\mbox{}\\

{
  \color{retryred}
  \textbf{A:} \texttt{examine desk 2}
  \par
  \textbf{O:}
  On the desk 2, you see a alarmclock 1,
  a bowl 3, and a pencil 1.
  \par
}\\[3pt]
\mbox{}\\

{
  \textbf{A:}
  \texttt{take pencil 1 from desk 2}
  \par
  \textbf{O:}
  You pick up the pencil 1 from the desk 2.
  \par
}\\[3pt]
\mbox{}\\

{
  \textbf{A:} \texttt{go to desk 1}
  \par
  \textbf{O:}
  On the desk 1, you see a bowl 2, a creditcard 2,
  a desklamp 1, a laptop 1, a mug 1, and a pen 2.
  \par
}\\[3pt]
\mbox{}\\

{
  \textbf{A:} \texttt{use desklamp 1}
  \par
  \textbf{O:}
  You turn on the desklamp 1.
  \par
}\\[3pt]

\end{longtable}

\begin{longtable}{
  @{}>{\raggedright\arraybackslash}p{\linewidth}@{}
}

\caption{Alternating-action example. A and O denote the action and its resulting observation. The repetition loop is marked in red including the pair of the loop.}
\label{tab:alfworld-retry-23}\\

\toprule
\textbf{Put a clean egg in microwave.} \\
\midrule
\endfirsthead

\toprule
\textbf{Put a clean egg in microwave.}
\textit{(continued)} \\
\midrule
\endhead

\midrule
\multicolumn{1}{r}{
  \textit{Continued on next page}
}\\
\endfoot

\bottomrule
\endlastfoot

{
  \textbf{A:} \texttt{go to fridge 1}
  \par
  \textbf{O:}
  The fridge 1 is closed.
  \par
}\\[3pt]
\mbox{}\\

{
  \textbf{A:} \texttt{open fridge 1}
  \par
  \textbf{O:}
  You open the fridge 1. The fridge 1 is open.
  In it, you see a apple 2, a egg 1, a lettuce 1,
  a pan 2, a plate 1, and a tomato 1.
  \par
}\\[3pt]
\mbox{}\\

{
  \textbf{A:} \texttt{take egg 1 from fridge 1}
  \par
  \textbf{O:}
  You pick up the egg 1 from the fridge 1.
  \par
}\\[3pt]
\mbox{}\\

{
  \textbf{A:} \texttt{put egg 1 in/on fridge 1}
  \par
  \textbf{O:}
  You put the egg 1 in/on the fridge 1.
  \par
}\\[3pt]
\mbox{}\\

{
  \color{retryred}
  \textbf{A:} \texttt{close fridge 1}
  \par
  \textbf{O:}
  You close the fridge 1.
  \par
}\\[3pt]
\mbox{}\\

{
  \color{retryred}
  \textbf{A:} \texttt{open fridge 1}
  \par
  \textbf{O:}
  You open the fridge 1. The fridge 1 is open.
  In it, you see a apple 2, a egg 1, a lettuce 1,
  a pan 2, a plate 1, and a tomato 1.
  \par
}\\[3pt]
\mbox{}\\

{
  \color{retryred}
  \textbf{A:} \texttt{close fridge 1}
  \par
  \textbf{O:}
  You close the fridge 1.
  \par
}\\[3pt]
\mbox{}\\

{
  \color{retryred}
  \textbf{A:} \texttt{open fridge 1}
  \par
  \textbf{O:}
  You open the fridge 1. The fridge 1 is open.
  In it, you see a apple 2, a egg 1, a lettuce 1,
  a pan 2, a plate 1, and a tomato 1.
  \par
}\\[3pt]
\mbox{}\\

{
  \color{retryred}
  \textbf{A:} \texttt{close fridge 1}
  \par
  \textbf{O:}
  You close the fridge 1.
  \par
}\\[3pt]
\mbox{}\\

{
  \color{retryred}
  \textbf{A:} \texttt{open fridge 1}
  \par
  \textbf{O:}
  You open the fridge 1. The fridge 1 is open.
  In it, you see a apple 2, a egg 1, a lettuce 1,
  a pan 2, a plate 1, and a tomato 1.
  \par
}\\[3pt]
\mbox{}\\

{
  \textbf{A:} \texttt{take egg 1 from fridge 1}
  \par
  \textbf{O:}
  You pick up the egg 1 from the fridge 1.
  \par
}\\[3pt]
\mbox{}\\

{
  \textbf{A:} \texttt{put egg 1 in/on fridge 1}
  \par
  \textbf{O:}
  You put the egg 1 in/on the fridge 1.
  \par
}\\[3pt]
\mbox{}\\

{
  \textbf{A:} \texttt{go to microwave 1}
  \par
  \textbf{O:}
  The microwave 1 is closed.
  \par
}\\[3pt]
\mbox{}\\

{
  \textbf{A:} \texttt{open microwave 1}
  \par
  \textbf{O:}
  You open the microwave 1.
  The microwave 1 is open.
  In it, you see a apple 1, and a potato 1.
  \par
}\\[3pt]
\mbox{}\\

{
  \textbf{A:} \texttt{take apple 1 from microwave 1}
  \par
  \textbf{O:}
  You pick up the apple 1 from the microwave 1.
  \par
}\\[3pt]
\mbox{}\\

{
  \textbf{A:} \texttt{go to fridge 1}
  \par
  \textbf{O:}
  The fridge 1 is open.
  In it, you see a apple 2, a egg 1, a lettuce 1,
  a pan 2, a plate 1, and a tomato 1.
  \par
}\\[3pt]
\mbox{}\\

{
  \textbf{A:} \texttt{put apple 1 in/on fridge 1}
  \par
  \textbf{O:}
  You put the apple 1 in/on the fridge 1.
  \par
}\\[3pt]
\mbox{}\\

{
  \textbf{A:} \texttt{take egg 1 from fridge 1}
  \par
  \textbf{O:}
  You pick up the egg 1 from the fridge 1.
  \par
}\\[3pt]
\mbox{}\\

{
  \textbf{A:} \texttt{put egg 1 in/on fridge 1}
  \par
  \textbf{O:}
  You put the egg 1 in/on the fridge 1.
  \par
}\\[3pt]
\mbox{}\\

{
  \textbf{A:} \texttt{go to microwave 1}
  \par
  \textbf{O:}
  The microwave 1 is open.
  In it, you see a potato 1.
  \par
}\\[3pt]
\mbox{}\\

{
  \textbf{A:} \texttt{take potato 1 from microwave 1}
  \par
  \textbf{O:}
  You pick up the potato 1 from the microwave 1.
  \par
}\\[3pt]
\mbox{}\\

{
  \textbf{A:} \texttt{put potato 1 in/on microwave 1}
  \par
  \textbf{O:}
  You put the potato 1 in/on the microwave 1.
  \par
}\\[3pt]
\mbox{}\\

{
  \textbf{A:} \texttt{examine microwave 1}
  \par
  \textbf{O:}
  The microwave 1 is open.
  In it, you see a potato 1.
  \par
}\\[3pt]
\mbox{}\\

{
  \textbf{A:} \texttt{go to fridge 1}
  \par
  \textbf{O:}
  The fridge 1 is open.
  In it, you see a apple 2, a apple 1, a egg 1,
  a lettuce 1, a pan 2, a plate 1, and a tomato 1.
  \par
}\\[3pt]
\mbox{}\\

{
  \textbf{A:} \texttt{take egg 1 from fridge 1}
  \par
  \textbf{O:}
  You pick up the egg 1 from the fridge 1.
  \par
}\\[3pt]
\mbox{}\\

{
  \textbf{A:} \texttt{put egg 1 in/on fridge 1}
  \par
  \textbf{O:}
  You put the egg 1 in/on the fridge 1.
  \par
}\\[3pt]
\mbox{}\\

{
  \textbf{A:} \texttt{go to microwave 1}
  \par
  \textbf{O:}
  The microwave 1 is open.
  In it, you see a potato 1.
  \par
}\\[3pt]
\mbox{}\\

{
  \textbf{A:} \texttt{close microwave 1}
  \par
  \textbf{O:}
  You close the microwave 1.
  \par
}\\[3pt]
\mbox{}\\

{
  \textbf{A:} \texttt{go to fridge 1}
  \par
  \textbf{O:}
  The fridge 1 is open.
  In it, you see a apple 2, a apple 1, a egg 1,
  a lettuce 1, a pan 2, a plate 1, and a tomato 1.
  \par
}\\[3pt]
\mbox{}\\

{
  \textbf{A:} \texttt{take egg 1 from fridge 1}
  \par
  \textbf{O:}
  You pick up the egg 1 from the fridge 1.
  \par
}\\[3pt]

\end{longtable}

\endgroup

\end{document}